\pdfoutput=1
\documentclass[10pt]{scrartcl}

\usepackage[english]{babel} 
\usepackage[protrusion=true,expansion=true]{microtype} 
\usepackage{amsmath,amsfonts,amssymb,amsthm,color,fullpage} 
\usepackage{verbatim,fancyvrb}
\usepackage{alltt}
\usepackage{tikz}
\usepackage{enumerate}
\usepackage{caption}
\usepackage{subfigure}
\usepackage[]{algorithm}
\usepackage{algpseudocode}
\usepackage{stmaryrd}
\usepackage{afterpage}
\usepackage{url}
\usepackage[english=american]{csquotes}
\usepackage{booktabs}
\usepackage{capt-of}

\newtheorem{definition}{Definition}
\newtheorem{theorem}[definition]{Theorem}

\newtheorem{proposition}[definition]{Proposition}
\newtheorem{lemma}[definition]{Lemma}
\newtheorem{example}[definition]{Example}

\newtheorem{notation}[definition]{Notation}

\def\B{{\mathbb B}}
\def\N{{\mathbb N}}

\def\R{{\mathbb R}}

\DeclareMathOperator{\dtw}{\delta}

\newcommand{\QED}{}
\newcommand{\commentout}[1]{}

\newcommand{\abs}[1]{\mathop{\left\lvert #1 \right\rvert}} 
\newcommand{\args}[1]{\mathop{\left( #1 \right)}} 

\newcommand{\norm}[1]{\mathop{\left\lVert #1 \right\rVert}}
\newcommand{\cbrace}[1]{\mathop{\left\{ #1 \right\}}}

\newcommand{\argsS}[2]{\mathop{\left( #1 \right)#2}} 

\newcommand{\normS}[2]{\mathop{\left\lVert #1 \right\rVert#2}}

\renewcommand{\S}[1]{{\mathcal{#1}}}           	

\renewenvironment{cases}{%
\left\{\begin{array}{c@{\quad : \quad}l}}%
{%
\end{array}\right.}

\begin{document}

\title{Semi-Metrification of the Dynamic Time Warping Distance}

\author{Brijnesh J.~Jain\\
 Technische Universit\"at Berlin, Germany\\
 e-mail: brijnesh.jain@gmail.com}
\date{}
\maketitle

\begin{abstract}
The dynamic time warping (dtw) distance fails to satisfy the triangle inequality and the identity of indiscernibles. As a consequence, the dtw-distance is not warping-invariant, which in turn results in peculiarities in data mining applications. This article converts the dtw-distance to a semi-metric and shows that its canonical extension is warping-invariant. Empirical results indicate that the nearest-neighbor classifier in the proposed semi-metric space performs comparably to the same classifier in the standard dtw-space. To overcome the undesirable peculiarities of dtw-spaces, this result suggests to further explore the semi-metric space for data mining applications.
\end{abstract}

\bigskip

\begin{small}
\fbox{\parbox{0.95\textwidth}{\tableofcontents}}
\end{small}
\clearpage


\section{Introduction}

Time series such as stock prices, weather data, biomedical measurements, and biometrics data are sequences of time-dependent observations. Comparing time series is a fundamental task in various data mining applications \cite{Aghabozorgi2015,Esling2012,Fu2011}. One challenge in comparing time series is to eliminate their temporal differences \cite{Sakoe1978}. A common and widely applied technique to deal with such temporal variation is the \emph{dynamic time warping} (dtw) distance \cite{Mueller2007}. 

The dtw-distance is not a metric. Consequently, dtw-spaces are mathematically less structured than metric spaces. To overcome the inherent structural limitations of the dtw-distance, metric distances and positive semi-definite alignment kernels have been proposed \cite{Abanda2018,Cuturi2011,Marteau2009}. The resulting (dis)similarities, however, differ from the original dtw-distance and therefore do not contribute to a better understanding of dtw-spaces. 

The missing metric properties of the dtw-distance are the triangle inequality and the identity of indiscernibles. The absence of the triangle inequality is widely acknowledged in the literature and has been analyzed theoretically \cite{Lemire2009} and empirically \cite{Casacuberta1987}. In contrast, the effects and possible peculiarities caused by the lack of the identity of indiscernibles in conjunction with the lack of the triangle inequality have not been clearly exposed in the literature.

\medskip

In this article, we convert the dtw-distance into a warping-invariant semi-metric to overcome peculiarities caused by the missing metric properties of the dtw-distance. A semi-metric is a distance that satisfies all axioms of a metric with exception of the triangle inequality. Since the proposed semi-metric is warping-invariant, it overcomes the peculiarities of the dtw-distance in data mining applications. In more detail:

\bigskip

\noindent
\emph{1. We define warping-invariance and show that the dtw-distance is not warping-invariant.}

Informally, warping-invariance refers to the property that a distance between two time series remains unchanged under compositions of compressions and expansions. The lack of warping-invariance is notable because the dtw-distance has been designed to overcome the inability of the Euclidean distance to cope with temporal variations. According to Kruskal and Liberman \cite{Kruskal1983}, the dtw-distance measures how different two time series are in a way that is insensitive to expansions and compressions but sensitive to other differences \cite{Kruskal1983}. According to Sakoe and Chiba \cite{Sakoe1978}, the dtw-distance eliminates timing difference by warping the time axis to minimize the accumulated differences between two time series. The limitation of the dtw-distance is that invariance under temporal transformations is only a pairwise but not a transitive property. 

\bigskip

\noindent
\emph{2. We show that the lack of warping-invariance causes peculiarities in data mining applications.}

The lack of warping-invariance results in a peculiar behavior of the nearest-neighbor rule. This peculiarity propagates to data mining methods in dtw-spaces based on the nearest-neighbor rule such as, $k$-nearest-neighbor classification \cite{Bagnall2017}, k-means clustering \cite{Abdulla2003,Cuturi2017,Hautamaki2008,Morel2018,Petitjean2016,Rabiner1979,Soheily-Khah2016}, learning vector quantization \cite{Jain2018,Somervuo1999}, and self-organizing maps \cite{Kohonen1998}. 

\bigskip

\noindent
\emph{3. We convert the dtw-distance into a warping-invariant semi-metric.} 

As illustrated in Figure \ref{fig:semi-metric}, the semi-metric is a dtw-distance defined on a subset of time series, called condensed forms. A condensed form is a time series without two consecutive identical elements. To compute the semi-metric of two time series, we first transform the underlying time series to their condensed forms by collapsing consecutive replicates to singletons. Then we apply the standard dtw-distance on the resulting condensed forms. Thus, the proposed semi-metric space is a dtw-space restricted to condensed forms. To convert the dtw-distance to a semi-metric, we develop a theoretical framework that combines novel results from words over arbitrary alphabets with warping walks -- a generalization of warping paths -- and matrix algebra. 

\begin{figure}
\centering
\includegraphics[width=0.8\textwidth]{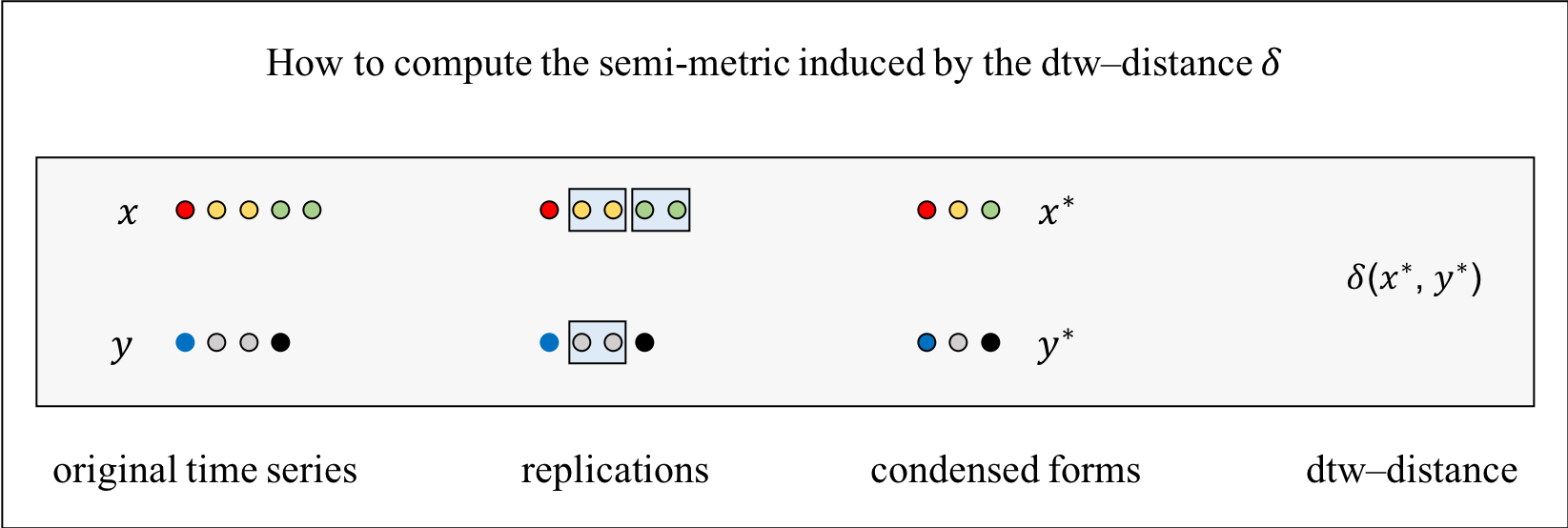}
\caption{Process of computing the proposed semi-metric. The first column shows two time series $x$ and $y$ consisting of $5$ and $4$ elements, respectively. Elements are shown by filled balls, where different colors refer to different real values. The second column identifies consecutive replications within the time series as highlighted by the blue-shaded boxes. The third column collapses replications to singletons resulting in condensed forms $x^*$ and $y^*$. The semi-metric is the dtw-distance between the condensed forms $x^*$ and $y^*$.}
\label{fig:semi-metric}
\end{figure}

\bigskip

\noindent
\emph{4. In experiments, we show that the error-rates of the nearest-neighbor classifiers in the semi-metric space and in the original dtw-space are comparable.}

This result suggests to further explore time series data mining methods in the proposed semi-metric space in order to overcome undesirable peculiarities caused by the standard dtw-distance.

\bigskip

The rest of this article is organized as follows: Section \ref{sec:results} informally outlines the basic approach and presents the main results. Section \ref{sec:peculiarities} discusses peculiarities caused by the lack of warping-invariance of the dtw-distance. Section \ref{sec:theory} develops the theoretical framework to prove the main results. Experiments are presented and discussed in Section \ref{sec:experiments}. Finally, Section \ref{sec:conclusion} concludes with a summary of the main findings and an outlook on further research.

\section{Results}\label{sec:results}

This section informally sketches the basic idea of the proposed approach and the main results. We begin with introducing the dtw-distance and defining warping-invariance. 

\subsection{The DTW-Distance}
Let $\N$ be the set of all positive integers and $\N_0$ is the set of all non-negative integers. We write $[n] = \cbrace{1, \ldots, n}$ for $n \in \N$. A real-valued \emph{time series} is a sequence $x = (x_1, \ldots, x_n)$ with elements $x_i \in \R$ for all $i \in [n]$. We denote the length of $x$ by $\abs{x}$ and the set of all real-valued time series of finite length by $\S{T}$. 

A \emph{warping path} of order $m \times n$ and length $\ell$ is a sequence $p = (p_1 , \dots, p_\ell)$ consisting of $\ell$ points $p_l = (i_l,j_l) \in [m] \times [n]$ such that
\begin{enumerate}
\item $p_1 = (1,1)$ and $p_\ell = (m,n)$ \hfill (\emph{boundary conditions})
\item $p_{l+1} - p_{l} \in \cbrace{(1,0), (0,1), (1,1)}$ for all $l \in [\ell-1]$ \hfill(\emph{step condition})
\end{enumerate}
We denote the set of all warping paths of order $m \times n$ by $\S{P}_{m,n}$. Suppose that $p = (p_1, \ldots, p_\ell) \in \S{P}_{m,n}$ is a warping path with $p_l = (i_l, j_l)$ for all $l \in [\ell]$. Then $p$ defines an expansion (warping) of time series $x = (x_1, \ldots, x_m)$ and $y = (y_1, \ldots, y_n)$ to time series $\phi_p(x) = (x_{i_1}, \ldots, x_{i_\ell})$ and $\psi_p(y) = (y_{j_1}, \ldots, y_{j_\ell})$ of the same length $\ell$. By definition, the length $\ell$ of a warping path satisfies $\max(m,n) \leq \ell < m+n$. This shows that $\ell \geq \max(m,n)$ and therefore $\phi_p(x)$ and $\psi_p(y)$ are indeed expansions of $x$ and $y$.

The \emph{cost} of warping time series $x$ and $y$ along warping path $p$ is defined by
\begin{equation*}
C_p(x, y) = \normS{\phi_p(x)-\psi_p(y)}{^2} = \sum_{(i,j) \in p} \argsS{x_i-y_j}{^2},
\end{equation*}
where $\norm{\cdot}$ denotes the Euclidean norm and $\phi_p$ and $\psi_p$ are the expansions defined by $p$. The \emph{dtw-distance} of $x$ and $y$ is of the form
\begin{align*}
\dtw(x, y) = \min \cbrace{\sqrt{C_p(x, y)} \,:\, p \in \S{P}_{m,n}}.
\end{align*}
A warping path $p$ with $C_p(x, y) = \dtw^2(x, y)$ is called an \emph{optimal warping path} of $x$ and $y$. By definition, the dtw-distance minimizes the Euclidean distance between all possible expansions that can be derived from warping paths. In addition, the dtw-distance satisfies the properties 
\begin{enumerate}
\itemsep0em
\item $\dtw(x, y) \geq 0$
\item $\dtw(x, x) = 0$
\item $d(x, y) = d(y, x)$
\end{enumerate} 
for all $x, y \in \S{T}$ and fails to satisfy the properties 
\begin{enumerate}
\itemsep0em
\item $\dtw(x, y) = 0 \; \Leftrightarrow \; x = y$ \hfill (identity of indiscernibles)
\item $\dtw(x, z) \leq \dtw(x, y) + \dtw(y, z)$ \hfill (triangle inequality)
\end{enumerate} 
for all $x, y, z \in \S{T}$. Computing the dtw-distance and deriving an optimal warping path is usually solved by applying techniques from dynamic programming \cite{Sakoe1978}.

\subsection{Warping-Invariance}

This section defines warping-invariance. Warping appears as expansion and compression with respect to the time axis. The dtw-distance measures how different two time series are in a way that is not sensitive to expansion-compression but sensitive to other differences \cite{Kruskal1983}. To define warping-invariance we need to introduce some concepts. For the sake of readability, Figure \ref{fig:expansions} illustrates most of the concepts we consider in this and the next section. 

\begin{figure}[t]
\centering
\includegraphics[width=0.8\textwidth]{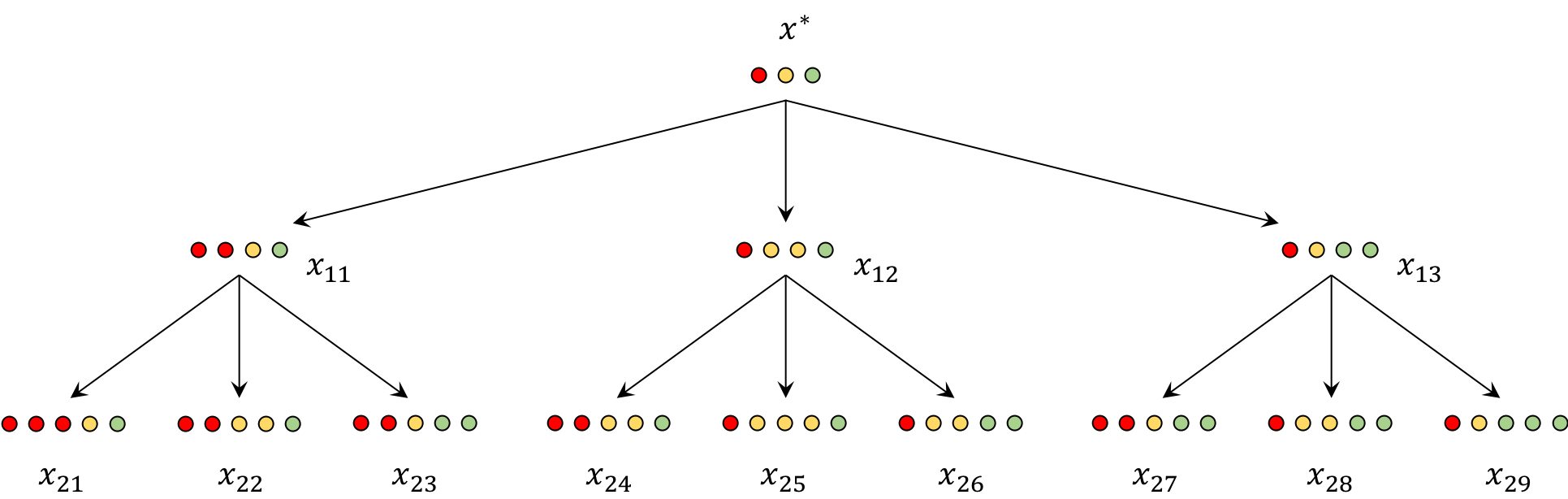}
\caption{Directed rooted tree with root $x^*$ of infinite depth. Nodes represent time series and edges represent one-element expansions. As in Figure \ref{fig:semi-metric}, elements of a time series are shown by filled balls, where different colors correspond to different values. The childs of a node $x$ are the one-element expansions of $x$. Every descendant $x'$ of a node $x$ is an expansion of $x$ (and $x$ is a compression of $x'$). If two nodes $x$ and $y$ have a common ancestor $z$, then $z$ is a common compression of $x$ and $y$. The root $x^*$ can not be expressed as an expansion of a shorter time series and is therefore irreducible. Thus, $x^*$ is a common compression and the condensed form of all nodes in the tree. }
\label{fig:expansions}
\end{figure}

An expansion of a time series is obtained by replicating a subset of its elements. More precisely, suppose that $x =(x_1, \ldots, x_n)$ is a time series. A time series $x'$ is an \emph{expansion} of time series $x$, written as $x' \succ x$, if there are indices $i_1, \ldots, i_k \in [n]$ and positive integers $r_1, \ldots, r_k \in \N$ for some $k \in \N_0$ such that $x'$ is of the form
\[
x' = (x_1, \ldots, 
\underbrace{x_{i_1}, \ldots, x_{i_1}}_{r_1\text{\,-times}}, x_{i_1 + 1}, \ldots, 
\underbrace{x_{i_2}, \ldots, x_{i_2}}_{r_2\text{\,-times}}, x_{i_2 + 1}, \ldots, 
\underbrace{x_{i_k}, \ldots, x_{i_k}}_{r_k\text{\,-times}}, x_{i_k+1}\ldots, x_n)
\]
By setting $k = 0$ we find that a time series is always an expansion of itself. A time series $x$ is a \emph{compression} of time series $x'$, written as $x \prec x'$, if $x'$ is an expansion of $x$. A time series $z$ is a \emph{common compression} of time series $x$ and $y$ if $z$ is a compression of $x$ and of $y$, that is $z \prec x$ and $z \prec y$. We write $x \sim y$ if both time series have a common compression. A distance function $d:\S{T} \times \S{T} \rightarrow \R_{\geq 0}$ is \emph{warping-invariant} if 
\[
d(x, y) = d(x', y')
\]
for all time series $x, y, x', y' \in \S{T}$ with $x \sim x'$ and $y \sim y'$. Figure \ref{eq:warping-identification} illustrates the concept of warping-invariance. The next example shows that the dtw-distance is not warping-invariant, although it has been designed to eliminate expansions and compressions.
\begin{example}\label{ex:warping-invariance} Let $x = (0, 1)$ and $x' = (0, 1, 1)$ two time series. Then $x'$ is an expansion of $x$ and $\dtw(x, x') = 0$. Suppose that $y = (0, 2)$ is another time series. Then we have
\begin{align*}
\dtw(x, y)^2 &= (0-0)^2+(1-2)^2 = 1\\
\dtw(x', y)^2 &= (0-0)^2+(1-2)^2+(1-2)^2 = 2.
\end{align*}
This implies $\dtw(x, y) \neq \dtw(x', y)$, where $x$ is a common compression of $x, x'$ and $y$ is a compression of itself. Hence, the dtw-distance is not warping-invariant. 
\qed
\end{example}

\subsection{Approach and Results}
The proposed approach to convert a dtw-space into a semi-metric space follows the standard approach to convert a pseudo-metric space to a metric space. A pseudo-metric is a distance that satisfies all axioms of a metric with exception of the identity of indiscernibles. However, further problems need to be resolved, because the dtw-distance additionally fails to satisfy the triangle inequality and therefore conveys less mathematical structure than a pseudo-metric. 

\medskip

Consider the relation $x \sim y \,\Leftrightarrow\, \dtw(x, y) = 0$ for all time series $x, y \in \S{T}$, called \emph{warping identification} henceforth. From Section \ref{sec:theory} follows that warping identification and common compression are equivalent definitions of the same relation. The next result forms the foundation for conversion of the dtw-distance to a semi-metric.

\paragraph*{\textbf{Proposition~\ref{prop:warping-identification-class}:}} {\em Warping identification is an equivalence relation on $\S{T}$.} \qed

\medskip

The assertion of Prop.~\ref{prop:warping-identification-class} is not self-evident, because warping paths are not closed under compositions and the dtw-distance fails to satisfy the triangle inequality. Therefore, it might be the case that there are time series $x, y,z$ such that $\delta(x, y) = \delta(y, z) = 0$ but $\delta(x, z) > \delta(x, y) + \delta(y, z) = 0$. 

Since warping identification $\sim$ is an equivalence relation, we can construct a quotient space. For every time series $x$ we define the equivalence class $[x] = \cbrace{y \in \S{T}\,:\, y \sim x}$ consisting of all time series that are warping identical to $x$. Then the set 
\[
\S{T}^* = \cbrace{[x] \,:\ x \in \S{T}}
\]
of all equivalence classes is the quotient space of $\S{T}$ by warping identification $\sim$. We may think of the quotient space as the space obtained by collapsing warping identical time series to a single point. Next, we endow the quotient space $\S{T}^*$ with the quotient distance
\[
\delta^*: \S{T}^* \times \S{T}^* \rightarrow \R_{\geq 0}, \quad ([x], [y]) \mapsto \inf_{x' \in [x]} \; \inf_{y' \in [y]} \; \delta(x', y').
\]

\medskip

The definition of the quotient distance $\dtw^*$ is posed as a solution to a min-min problem of dtw-distances over infinite sets. This formulation is inconvenient for theoretical and practical purposes. To show that $\dtw^*$ is a semi-metric that can be efficiently computed, we derive an equivalent formulation of $\dtw^*$. For this, we introduce irreducible time series. A time series is said to be \emph{irreducible} if it cannot be expressed as an expansion of a shorter time series. The next result shows that an equivalence class can be generated by expanding an irreducible time series in all possible ways.

\paragraph*{\textbf{Proposition~\ref{prop:generator-of-[x]}:}} \emph{For every time series $x$ there is an irreducible time series $x^*$ such that} 
\begin{align*}
[x] &= \cbrace{y \in \S{T} \,:\, y \succ x^*}.
\end{align*}

\medskip

The irreducible time series $x^*$ in Prop.~\ref{prop:generator-of-[x]} is called the \emph{condensed form} of $x$. Every time series has a unique condensed form by Prop.~\ref{prop:existence-of-condensed-form}. Hence, $x^*$ is the condensed form of all time series contained in the equivalence class $[x]$. Proposition~\ref{prop:generator-of-[x]} states that the equivalence classes are the trees rooted at a their respective condensed forms as illustrated in Figure \ref{fig:expansions}. Next, we show that expansions do not decrease the dtw-distance to other time series. 

\paragraph*{\textbf{Proposition~\ref{prop:expansion-inequality}:}} \emph{Let $x, x' \in \S{T}$ such that $x' \succ x$. Then $\dtw(x', z) \geq \dtw(x,z)$ for all $z \in \S{T}$.\qed}

\medskip

We will use Prop.~\ref{prop:expansion-inequality} in Section \ref{sec:peculiarities} to discuss peculiarities of data-mining applications on time series. By invoking Prop.~\ref{prop:generator-of-[x]} and Prop.~\ref{prop:expansion-inequality} we obtain the first key result of this contribution.

\paragraph*{\textbf{Theorem \ref{theorem:semi-metric}:}} \emph{The quotient distance $\dtw^*$ induced by the dtw-distance $\dtw$ is a well-defined semi-metric satisfying 
\[
\dtw^*([x], [y]) = \dtw(x^*,y^*)
\]
for all $x, y \in \S{T}$ with condensed forms $x^*$ and $y^*$, respectively. 
} 

\medskip

Theorem \ref{theorem:semi-metric} converts the dtw-distance to a semi-metric that satisfies the identity of indiscernibles and shows that the quotient distance $\dtw^*([x], [y])$ can be efficiently computed by first compressing $x$ and $y$ to their condensed forms $x^*$ and $y^*$, resp., and then taking their dtw-distance $\dtw(x^*,y^*)$.

The second key result of this contribution shows that the semi-metric induced by the dtw-distance is warping-invariant. Since warping-invariance is defined on the set $\S{T}$ rather than on the quotient space $\S{T}^*$, we extend the quotient distance $\dtw^*$ to a distance on $\S{T}$ by virtue of
\[
\dtw^\sim: \S{T} \times \S{T} \rightarrow \R_{\geq 0}, \quad (x, y) \mapsto \dtw^\sim(x, y) = \dtw^*([x], [y]).
\]
We call $\dtw^\sim$ the \emph{canonical extension} of $\dtw^*$.

\paragraph*{\textbf{Theorem \ref{theorem:warping-invariance}:}} \emph{The canonical extension $\dtw^\sim$ of the quotient distance $\dtw^*$ is warping-invariant.
} 

\section{Peculiarities of the DTW-Distance}\label{sec:peculiarities}

This section discusses peculiarities of data-mining applications on time series caused by the absence of the triangle inequality and the identity of indiscernibles of the dtw-distance. 

\subsection{The Effect of Expansions on the Nearest-Neighbor Rule}

Suppose that $x'$ is an expansion of $x$. Then Prop.~\ref{prop:expansion-inequality} yields
\begin{align}
\label{eq:warping-identification}
\dtw(x, x') &= 0\\
\label{eq:expansion-inequality}
\dtw(x, y)\phantom{'} &\leq \dtw(x', y)
\end{align}
for all time series $y$. Equation \eqref{eq:warping-identification} states that a time series is warping identical with its expansions. Equation \eqref{eq:expansion-inequality} says that expansions do not decrease the dtw-distance to other time series. Observe that Eq.~\eqref{eq:warping-identification} and \eqref{eq:expansion-inequality} describe the situation of Example \ref{ex:warping-invariance} in general terms. 

Equation \eqref{eq:expansion-inequality} affects the nearest-neighbor rule. To see this, we consider a set $\S{D} = \cbrace{x, y} \subseteq \S{T}$ of two prototypes. The Voronoi cells of $\S{D}$ are defined by 
\begin{align*}
\S{V_D}(x) &= \cbrace{z \in \S{T} \,:\, \dtw(x, z) \leq \dtw(y, z)}\\
\S{V_D}(y) &= \cbrace{z \in \S{T} \,:\, \dtw(y, z) \leq \dtw(x, z)}
\end{align*}
The nearest neighbor of a time series $z$ is prototype $x$ if $z \in \S{V_D}(x)$ and prototype $y$ otherwise. The nearest neighbor rule assigns $z$ to its nearest neighbor. To break ties, we arbitrarily assign $x$ as nearest neighbor for all time series residing on the boundary 
\[
\S{B_D}(x,y) = \S{V_D}(x) \cap \S{V_D}(y) = \cbrace{z \in \S{T} \,:\, \dtw(x, z) = \dtw(y, z)}.
\]
Suppose that we replace prototype $x$ by an expansion $x'$ to obtain a modified set $\S{D}' = \cbrace{x', y}$ of prototypes. Although $x'$ and $x$ are warping identical by Eq.~\eqref{eq:warping-identification}, the Voronoi cells of $\S{D}$ and $\S{D}'$ differ. From Eq.~\eqref{eq:expansion-inequality} follows that $\S{V_{D'}}(x') \subseteq \S{V_D}(x)$, which in turn implies $\S{V_{D'}}(y) \supseteq \S{V_D}(y)$. Keeping prototype $y$ fixed, expansion of prototype $x$ decreases its Voronoi cell and thereby increases the Voronoi cell of prototype $y$. This shows that the nearest-neighbor rule in dtw-spaces depends on temporal variation of the prototypes. 

This peculiarity of the nearest-neighbor rule affects data mining methods on time series such as $k$-nearest-neighbor classification \cite{Bagnall2017}, k-means clustering \cite{Abdulla2003,Cuturi2017,Hautamaki2008,Morel2018,Petitjean2016,Rabiner1979,Soheily-Khah2016}, learning vector quantization \cite{Jain2018,Somervuo1999}, and self-organizing maps \cite{Kohonen1998}.

\subsection{The Effect of Expansions on k-Means}

In this section, we discuss peculiarities of the k-means algorithm in dtw-spaces as an example application of the nearest-neighbor rule. To apply k-means in dtw-spaces, we need a concept of average of time series. Different forms of time series averages have been proposed (see \cite{Schultz2018} for an overview). One principled formulation of an average is based on the notion of Fr\'echet function \cite{Frechet1948}: Suppose that $\S{S} = \args{x_1, \dots, x_n}$ is a sample of $n$ time series $x_i \in \S{T}$. Then the \emph{Fr\'echet function} of $\S{S}$ is defined by
\[
F: \S{T} \rightarrow \R, \quad z \mapsto \sum_{i=1}^n \dtw\!\argsS{x_i, z}{^2},
\]
A \emph{sample mean} of $\S{S}$ is any time series $\mu \in \S{T}$ that satisfies $F(\mu) \leq F(z)$ for all $z \in \S{T}$. A sample mean exists but is not unique in general \cite{Jain2016b}. Computing a mean of a sample of time series is NP-hard \cite{Bulteau2018}. Efficient heuristics to approximate a mean of a fixed and pre-specified length are the stochastic subgradient method \cite{Schultz2018}, soft-dtw \cite{Cuturi2017}, and a majorize-minimize algorithm \cite{Hautamaki2008,Petitjean2011} that has been popularized by \cite{Petitjean2011} under the name DTW Barycenter Averaging algorithm. 

\medskip

The k-means algorithm can be generalized to dtw-spaces by replacing arithmetic means of vectors with sample means of time series. A partition of a set $\S{S}= \cbrace{x_1, \dots, x_n} \subseteq \S{T}$ of time series is a set $\S{C} = \args{\S{C}_1, \ldots, \S{C}_k}$ of $k$ non-empty subsets $\S{C}_i \subseteq \S{S}$, called \emph{clusters}, such that $\S{S}$ is the disjoint union of these clusters. By $\Pi_k$ we denote the set of all partitions consisting of $k$ clusters. The objective of k-means is to minimize the cost function
\[
J: \Pi_k \rightarrow \R, \quad \args{\S{C}_1, \ldots, \S{C}_k} \mapsto \sum_{i=1}^k \sum_{x \in \S{C}_i} \dtw(x, \mu_i)^2,
\]
where $\mu_i$ is a mean of cluster $\S{C}_i$ for all $i \in [k]$. We can equivalently express the cost function $J$ by
\[
J \args{\S{C}_1, \ldots, \S{C}_k} = \sum_{i=1}^k F_i(\mu_i)
\]
where $F_i$ is the Fr\'echet function of cluster $\S{C}_i$. To minimize the cost function $J$, the standard k-means algorithm starts with an initial set $\mu_1, \ldots, \mu_k$ of means and then proceeds by alternating between two steps until termination:
\begin{enumerate}
\itemsep0em
\item \emph{Assignment step}: Assign each sample time series to the cluster of its closest mean. 
\item \emph{Update step}: Recompute the means for every cluster. 
\end{enumerate}
Due to non-uniqueness of sample means, the performance of k-means does not only depend on the choice of initial centroids as in Euclidean spaces, but also on the choice of centroids in the update step. In the following, we present two examples of peculiar behavior of k-means in dtw-spaces. 

\begin{example}\label{example:assignment-k-means}\em
We assume that the four time series 
\begin{align*}
x_1 &= (-1, 0, 0) & x_3 &= (0, 2, 3)\\
x_2 &= (-1, 0, 2) & x_4 &= (1, 2, 3)
\end{align*}
are partitioned into two clusters $\S{C}_1 = \cbrace{x_1, x_2}$ and $\S{C}_2 = \cbrace{x_3, x_4}$. Cluster $\S{C}_1$ has a unique mean $\mu_1 = (-1,0,1)$. Cluster $\S{C}_2$ has infinitely many means as indicated by Figure \ref{fig:non-condensed}. For every $r \in \N$ the time series 
\[
\mu_2^r = (0.5, 2, \underbrace{3, \ldots, 3}_{r-\text{times}}),
\]
is a mean of $\S{C}_2$. From Eq.~\eqref{eq:warping-identification} and the transitivity of expansions follows that $\mu_2^r$ and $\mu_2^s$ are warping identical for all $r, s \in \N$. Equation \eqref{eq:expansion-inequality} yields $\dtw(\mu_2^1, y) \leq \cdots \leq \dtw(\mu_2^r, y)$ for all time series $y \in \S{T}$ and all $r \in \N$. Hence, with increasing number $r$ of replications, the Voronoi cell of site $\mu_2^r$ decreases, whereas the Voronoi cell of site $\mu_1$ increases. 
\qed
\end{example}

\medskip

Example \ref{example:assignment-k-means} shows that the Voronoi cell of a centroid $\mu_i$ can be externally controlled by expanding or compressing $\mu_i$ without changing the Fr\'echet variation $F(\mu_i)$. This in turn affects the assignment step, which is based on the nearest-neighbor rule. 

\begin{figure*}[t]
\centering
 \includegraphics[width=0.85\textwidth]{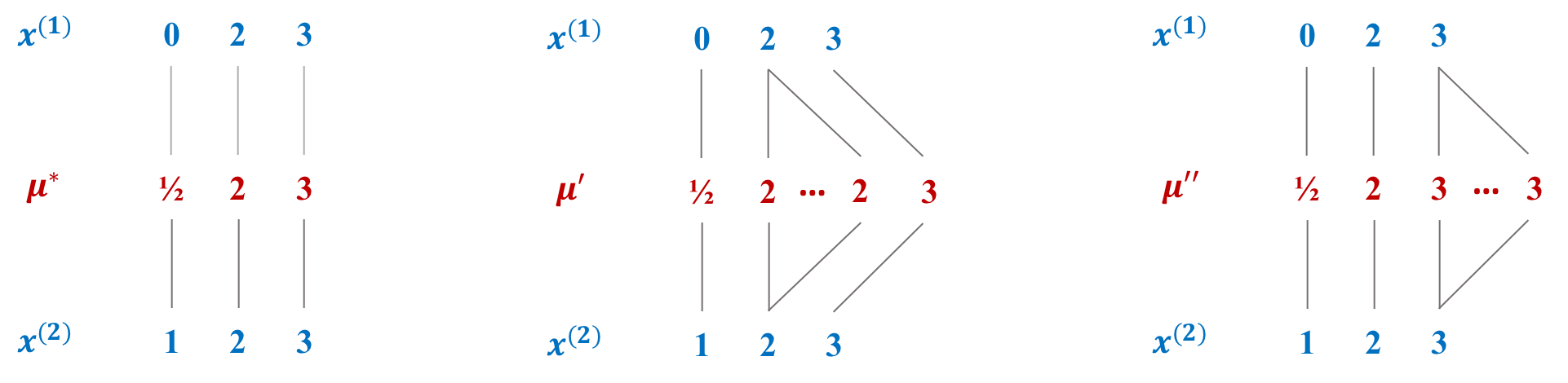}
 \caption{Means $\mu^*, \mu', \mu''$ of sample $\S{S} = \args{x^{(1)}, x^{(2)}}$. The means $\mu'$ and $\mu''$ are expansions of $\mu^*$.}
\label{fig:non-condensed}
\end{figure*}

\begin{example}\label{example:non-condensed-k-means}\em
Consider the two clusters of the four time series from Example \ref{example:assignment-k-means}. One way to measure the quality of the clustering $\S{C} = \cbrace{\S{C}_1, \S{C}_2}$ is by \emph{cluster cohesion} and \emph{cluster separation}. Cluster cohesion is typically defined by
\[
F_{\text{cohesion}}(\mu_1, \mu_2) = F_1(\mu_1) + F_2(\mu_2),
\]
where $F_1(z)$ and $F_2(z)$ are the Fr\'echet functions of $\S{C}_1$ and $\S{C}_2$, respectively. Thus, cluster cohesion sums the Fr\'echet variations within each cluster. Cluster separation can be defined by
\[
F_{\text{separation}}(\mu_1, \mu_2) = \dtw(\mu_1, \mu_2)^2.
\]
Cluster separation measures how well-separated or distinct two clusters are. A good clustering simultaneously minimizes cluster cohesion and maximizes cluster separation. Cluster cohesion is invariant under the choice of cluster means, because the Fr\'echet variations $F_1(\mu_1)$ and $F_2(\mu_2)$ are well-defined. The situation is different for cluster separation. From Eq.~\eqref{eq:warping-identification} follows $\dtw(\mu_1, \mu_2^1)^2 \leq \cdots \leq \dtw(\mu_1, \mu_2^r)$ for all $r \in \N$. Thus, cluster separation depends on the choice of mean $\mu_2$ of the second cluster $\S{C}_2$ (the mean $\mu_1$ of $\S{C}_1$ is unique). Figure \ref{fig:clus_sep} shows that cluster separation linearly increases with increasing number $r$ of replications. 
\begin{figure}[t]
\centering
\includegraphics[width=0.6\textwidth]{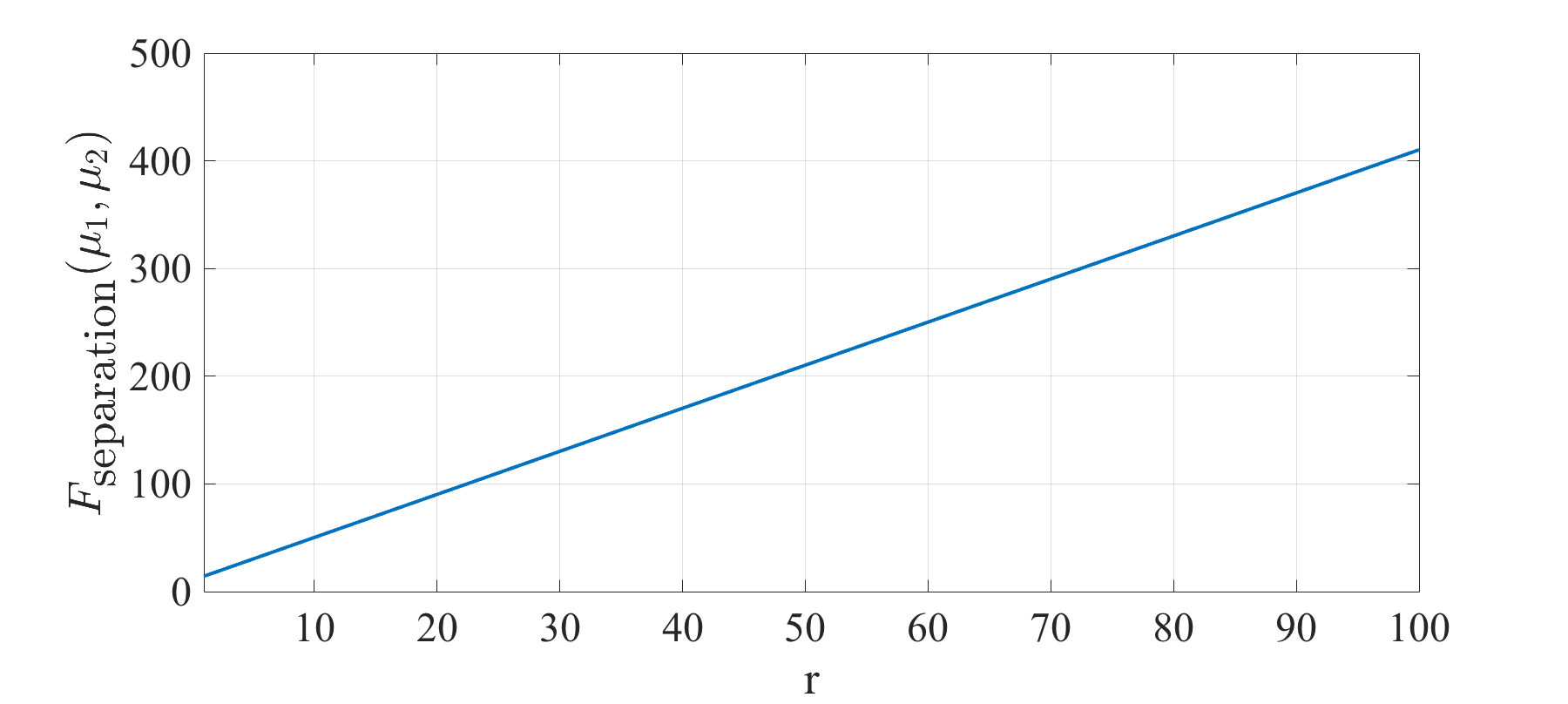}
\caption{Cluster separation $F_{\text{separation}}(\mu_1, \mu_2)$ as a function of the number $r$ of replications of the third element of $\mu^*_1 = (0.5, 2, 3)$.}
\label{fig:clus_sep}
\end{figure}
\qed
\end{example}

\medskip

While cluster cohesion is independent of the choice of mean, cluster separation can be maximized to infinity by expansions. Thus, cluster quality using cohesion and separation as defined in Example \ref{example:non-condensed-k-means} is not an inherent property of a clustering but rather a property of the chosen means as centroids, whose lengths can be controlled externally.

\section{Semi-Metrification of DTW-Spaces}\label{sec:theory}

This section converts the dtw-space to a semi-metric space and shows that the canonical extension of the proposed semi-metric is warping-invariant. The technical treatment is structured as follows: Section \ref{subsec:words} studies expansions, compressions, condensed forms, and irreducibility in terms of words over arbitrary alphabets. Sections \ref{subsec:warping-walks} and \ref{subsec:properties-of-warping-walks} propose warping walks as a more general and convenient concept than warping paths. In addition, we study some properties of warping walks. Section \ref{subsec:semi-metrification} constructs a semi-metric quotient space induced by the dtw-space as described in Section \ref{sec:results}. Finally, Section \ref{subsec:warping-invariance} proves that the canonical extension of the quotient metric is warping-invariant. 

For the sake of readability, this section is self-contained and restates all definitions and notations mentioned in earlier sections. Throughout this section, we use the following notations: 

\begin{notation}
The set $\R_{\geq 0}$ is the set of non-negative reals, $\N$ is the set of positive integers, and $\N_0$ is the set of non-negative integers. We write $[n] = \cbrace{1, \ldots, n}$ for $n \in \N$.
\end{notation}

\subsection{Prime Factorizations, Expansions, and Condensed Forms of Words}\label{subsec:words}

The key results of this article apply the same auxiliary results on warping paths and on time series. Both, warping paths and time series, can be regarded as words over different alphabets. We derive the auxiliary results on the more general structure of words. We first propose the notions of prime words and prime factorization of words. Then we show that every word has a unique prime factorization. Using prime factorizations we define expansions, compressions, condensed forms, and irreducible words. Finally, we study the relationships between these different concepts. 

\medskip

Let $\S{A}$ be a set, called \emph{alphabet}. We impose no restrictions on the set $\S{A}$. Thus, the elements of $\S{A}$ can be symbols, strings, trees, graphs, functions, reals, feature vectors, and so on. A \emph{word} over $\S{A}$ is a finite sequence $x = x_1 \cdots x_n$ with elements $x_i \in \S{A}$ for all $i \in [n]$. The set $\S{A}(x) = \cbrace{x_1, \ldots, x_n}$ is the set of elements contained in $x$. We denote by $\S{A}^*$ the set of all words over $\S{A}$ and by $\varepsilon$ the empty word. We write $\S{A}^+ = \S{A} \setminus \cbrace{\varepsilon}$ to denote the set of non-empty words over $\S{A}$.

Let $x = x_1 \cdots x_m$ and $y = y_1 \cdots y_n$ be two words over $\S{A}$. The \emph{concatenation} $z = xy$ is a word $z = z_1 \cdots z_{m+n}$ with elements
\[
z_i = \begin{cases}
x_i & 1 \leq i \leq m\\
y_{m-i} & m < i \leq m+n
\end{cases}
\]
for all $i \in [m+n]$. The concatenation is an associative operation on $\S{A}^*$. Hence, we can omit parentheses and write $xyz$ instead of $(xy)z$ or $x(yz)$ for all $x, y, z \in \S{A}^*$. For a given $n \in \N_0$, we write $x^n$ to denote the $n$-fold concatenation of a word $x \in \S{A}^*$. We set $x^0 = \varepsilon$ for every $x \in \S{A}^*$ and $\varepsilon^n = \varepsilon$ for all $n \in \N_0$. 

The \emph{length} of a word $x$ over $\S{A}$, denoted by $\abs{x}$, is the number of its elements. A \emph{prime word} is a word of the form $x = a^n$, where $a \in \S{A}$ and $n \in \N$ is a positive integer. From the definition follows that $\S{A}(a^n) = \cbrace{a}$ and that the empty word $\varepsilon$ is not a prime word. 

A non-empty word $v \in \S{A}^+$ is a \emph{factor} of a word $x \in \S{A}^*$ if there are words $u, w \in \S{A}^*$ such that $x = uvw$. Every non-empty word $x$ is a factor of itself, because $x = \varepsilon x \varepsilon$. In addition, the empty word $\varepsilon$ has no factors. 

Let $\S{A}^{**} = (\S{A}^*)^*$ be the set of all finite words over the alphabet $\S{A}^*$. A \emph{prime factorization} of a word $x \in \S{A}^*$ is a word $\Pi(x) = p_1 \cdots p_k \in \S{A}^{**}$ of prime factors $p_t \in \S{A}^*$ of $x$ for all $t \in [k]$ such that 
\begin{enumerate}
\itemsep0em
\item $p_1 \cdots p_k = x$ \hfill (\emph{partition-condition})
\item $\S{A}(p_t) \neq \S{A}(p_{t+1})$ for all $t \in [k-1]$ \hfill (\emph{maximality-condition})
\end{enumerate}
The partition-condition demands that the concatenation of all prime factors yields $x$. The maximality-condition demands that a prime factorization of $x$ consists of maximal prime factors. By definition, we have $\Pi(\varepsilon) = \varepsilon$.

\begin{proposition}
Every word over $\S{A}$ has a unique prime factorization.
\end{proposition}

\begin{proof}
Since $\Pi(\varepsilon) = \varepsilon$, it is sufficient to consider non-empty words over $\S{A}$. We first show that every $x \in \S{A}^+$ has a prime factorization by induction over the length $\abs{x} = n$. 

\medskip

\noindent
\emph{Base case}: Suppose that $n = 1$. Let $x$ be a word over $\S{A}$ of length one. Then $x$ is prime and a factor of itself. Hence, $\Pi(x) = x$ is a prime factorization of $x$.

\medskip

\noindent
\emph{Inductive step}: Suppose that every word over $\S{A}$ of length $n$ has a prime factorization. Let $x = x_1 \cdots x_{n+1}$ be a word of length $\abs{x} = n+1$. By induction hypothesis, the word $x' = x_1 \cdots x_n$ has a prime factorization $\Pi(x') = p_1 \cdots p_k$ for some $k \in [n]$. Suppose that $p_k = a^{n_k}$ for some $n_k \in \N$. Then $\S{A}(p_k) = a$. We distinguish between two cases: 
\begin{enumerate}
\itemsep0em
\item 
Case $a = x_{n+1}$: The concatenation $\tilde{p}_k = p_k a = a^{n_k + 1}$ is a prime word with $\S{A}(\tilde{p}_k) = \cbrace{a}$. Let $q = p_1\cdots p_{k-1}$. We have 
\[
x = x'a = p_1 \cdots p_k a \stackrel{(1)}{=} p_1 \cdots p_{k-1}\tilde{p}_k \stackrel{(2)}{=} q \tilde{p}_k \varepsilon.
\] 
Equation (1) shows the first property of a prime factorization and equation (2) shows that $\tilde{p}_k$ is a factor of $x$. From the induction hypothesis follows that $\S{A}(p_{k-1}) \neq \S{A}(p_k) = \S{A}(\tilde{p}_k)$. This shows that $\Pi(x) = p_1 \cdots p_{k-1} \tilde{p}_k$ is a prime factorization of $x$. 
\item 
Case $a \neq x_{n+1}$: The word $p_{k+1} = x_{n+1}$ is a prime factor of $x$ satisfying $p_1 \cdots p_k p_{k+1} = x$ and 
$\S{A}(p_k) = a \neq x_{n+1} = \S{A}(p_{k+1})$. Then from the induction hypothesis follows that $\Pi(x) = p_1 \cdots p_{k+1}$ is a prime factorization of $x$.
\end{enumerate}

\medskip

It remains to show that a prime factorization is unique. Suppose that $\Pi(x) = p_1 \cdots p_{k}$ and $\Pi'(x) = q_1 \cdots q_l$ are different prime factorizations of $x$. Suppose that the prime factors are of the form $p_r = a_r^{m_r}$ and $q_s = b_s^{n_s}$ for all $r \in [k]$ and $s \in [l]$, where $a_r, b_s \in \S{A}$ and $m_r, n_s \in \N$. The partition-condition of a prime factorization yields
\[
x = a_1^{m_1} \cdots a_k^{m_k} = b_1^{n_1} \cdots b_l^{n_l}.
\]
Since $\Pi(x)$ and $\Pi'(x)$ are different, we can find a smallest index $i \in [k] \cap [l]$ such that $p_i \neq q_i$. 
From $p_i \neq q_i$ follows that $a_i^{m_i} \neq b_i^{n_i}$. Since $i$ is the smallest index for which the prime factors of both prime factorizations differ, we have
\[
u = a_1^{m_1} \cdots a_{i-1}^{m_{i-1}} = b_1^{n_1} \cdots b_{i-1}^{n_{i-1}}.
\]
There are words $v, w \in \S{A}^*$ such that $x = u a_i v = u b_i w = x$. This shows that $a_i = b_i$. Hence, from $a_i^{m_i} \neq b_i^{n_i}$ and $a_i = b_i$ follows that $m_i \neq n_i$. We distinguish between two cases: (i) $m_i < n_i$ and (ii) $m_i > n_i$. We assume that $m_i < n_i$. Suppose that
\[
u = a_1^{m_1} \cdots a_{i-1}^{m_{i-1}}a_i^{m_i} = b_1^{n_1} \cdots b_{i-1}^{n_{i-1}}b_i^{m_i}.
\]
From $m_i < n_i$ follows $\abs{u} < \abs{x}$. This implies that $i < k$, that is $p_{i+1}$ is an element of $\Pi(x)$. Thus, there are words $v, w \in \S{A}^*$ such that $x = u a_{i+1} v = u b_i w$. We obtain 
\[
\S{A}(p_{i+1}) = a_{i+1} = b_i = a_i = \S{A}(p_i). 
\]
This violates the maximality-condition of a prime factorization and therefore contradicts our assumption that $m_i < n_i$. In a similar way, we can derive a contradiction for the second case $m_i > n_i$. Combining the results of both cases gives $m_i = n_i$, which contradicts our assumption that both prime factorizations are different. Thus a prime factorization is uniquely determined. This completes the proof. 
\end{proof}

\medskip

Let $x$ and $x'$ be words over $\S{A}$ whose prime factorizations $\Pi(x) = p_1 \cdots p_k$ and $\Pi(y) = p'_1 \cdots p'_k$, resp., have the same length $k$. We say $x$ is an \emph{expansion} of $x'$, written as $x \succ x'$, if $\S{A}(p_t) = \S{A}(p'_t)$ and $\abs{p_t} \geq \abs{p'_t}$ for all $t \in [k]$. If $x$ is an expansion of $x'$, then $x'$ is called a \emph{compression} of $x$, denoted by $x' \prec x$. By definition, every word is an expansion (compression) of itself. The set 
\[
\S{C}(x) = \cbrace{x' \in \S{A}^* \,:\, x' \prec x}
\]
is the set of all compressions of $x$. Suppose that $x_0, x_1, \ldots, x_k$ are $k+1$ time series. Occasionally, we write $x_0 \succ x_1, \ldots, x_k$ for $x_0 \succ x_1, \ldots, x_0 \succ x_k$ and similarly $x_0 \prec x_1, \ldots, x_k$ for $x_0 \prec x_1, \ldots, x_0 \prec x_k$.

An \emph{irreducible word} is a word $x = x_1 \cdots x_n$ over $\S{A}$ with prime factorization $\Pi(x) = x_1 \cdots x_n$. From $\Pi(\varepsilon) = \varepsilon$ follows that the empty word $\varepsilon$ is irreducible. A \emph{condensed form} of $x$ a word over $\S{A}$ is an irreducible word $x^*$ such that $x \succ x^*$. The condensed form of $\varepsilon$ is $\varepsilon$. We show that every word has a uniquely determined condensed form.

\begin{proposition}\label{prop:existence-of-condensed-form}
Every word has a unique condensed form. 
\end{proposition}

\begin{proof}
It is sufficient to show existence and uniqueness of a condensed form of non-empty words. Let $x$ be a word over $\S{A}$ with prime factorization $\Pi(x) = p_1 \cdots p_k$. Let $\S{A}(p_t) = \cbrace{a_t}$ for all $t \in [k]$. We define the word $x^* = a_1 \cdots a_k$. Then $\Pi(x^*) = a_1 \cdots a_k$ is the prime factorization of $x^*$ that obviously satisfies the partition-condition and whose maximality-condition is inherited by the maximality-condition of the prime factorization $\Pi(x)$. This shows that $x^*$ is irreducible and a compression of $x$. Hence, $x^*$ is a condensed form of $x$. 

Suppose that $z = b_1 \cdots b_l$ is an irreducible word such that $z \prec x$. Then we have $k = l$, $\S{A}(p_t) = \S{A}(b_t)$, and $\abs{p_t} \geq \abs{b_t} = 1$ for all $t \in [k]$. Observe that $\cbrace{a_t} = \S{A}(p_t) = \S{A}(b_t) = \cbrace{b_t}$. We obtain $z = a_1 \cdots a_k$ showing the uniqueness of the condensed form of $x$. 
\QED\end{proof}

\medskip

The next result shows that expansions (compressions) are transitive. 
\begin{proposition}\label{prop:expansions-are-transitive}
Let $x, y, z \in \S{A}^*$. From $x \succ y$ and $y \succ z$ follows $x \succ z$. 
\end{proposition}

\begin{proof}
Let $\Pi(x) = u_1 \cdots u_k$, $\Pi(y) = v_1 \cdots v_l$, and $\Pi(z) = w_1 \cdots w_m$ be the prime factorizations of $x$, $y$, and $z$, respectively. From $x \succ y$ and $y \succ z$ follows that $k = l = m$. Let $t \in [k]$. We have $u_t \succ v_t$ and $v_t \succ w_t$. From $u_t \succ v_t$ follows that there are elements $a_t \in \S{A}$ and $\alpha_t, \beta_t \in \N$ with $\alpha_t \geq \beta_t$ such that $u_t = a_t^{\alpha_t}$ and $v_t = a_t^{\beta_t}$. From $v_t \succ w_t$ follows that there is an element $\gamma_t \in \N$ with $\beta_t \geq \gamma_t$ such that $w_t = a_t^{\gamma_t}$. Since $\alpha_t \geq \beta_t$ and $\beta_t \geq \gamma_t$, we obtain $u_t \succ w_t$. We have chosen $t \in [k]$ arbitrarily. Hence, we find that $x \succ z$. This proves the assertion.
\QED\end{proof}

\medskip

Expansions have been introduced as expansions on the prime factors. Lemma \ref{lemma:expansions} states that expansions are obtained by replicating a subset of elements of a given word.

\begin{lemma}\label{lemma:expansions}
Let $x = x_1 \cdots x_n$ and $y = y_1 \cdots y_m$ be words over $\S{A}$. Then the following statements are equivalent:
\begin{enumerate}
\itemsep0em
\item
$x$ is an expansion of $y$.
\item 
There are $\alpha_1, \ldots, \alpha_m \in \N$ such that $x = y_1^{\alpha_1} \cdots y_m^{\alpha_m}$.
\end{enumerate}
\end{lemma}

\begin{proof} 
Suppose that $\Pi(x) = p_1 \cdots p_k$ and $\Pi(y) = q_1 \cdots q_l$ are the prime factorization of $x$ and $y$, respectively. 

\medskip

\noindent
$\Rightarrow$: We assume that $x$ is an expansion of $y$. Then from $x \succ y$ follows that $k = l$, $\S{A}(p_t) = \S{A}(q_t)$, and $\abs{p_t} \geq \abs{q_t}$ for all $t \in [k]$. We arbitrarily pick an element $t \in [k]$. There are elements $a_t \in \S{A}$ and $\beta_t, \gamma_t \in \N$ such that $p_t = a_t^{\beta_t}$, $q_t = a_t^{\gamma_t}$ and $\beta_t \geq \gamma_t$. In addition, there is an index $i_t \in [m]$ such that 
\begin{align*}
q_t = y_{i_t} y_{i_t+1}\cdots y_{i_t+\gamma_t-1} = a_t^{\gamma_t}.
\end{align*}
Let $\nu_t = \beta_t - \gamma_t -1$. Then $\nu_t \geq 0$ and we have
\[
p_t = y_{i_t}^{\nu_t} y_{i_t+1}^1 \cdots y_{i_t+\gamma_t-1}^{1} = a_t^{\beta_t}.
\]
We set $\alpha_{i_t} = \nu_t$ and $\alpha_{i_t+1} = \cdots \alpha_{i_t+\gamma_t-1} = 1$. Concatenating all prime factors $p_t$ of $x$ yields the assertion. 

\medskip

\noindent
$\Leftarrow$: We assume that there are integers $\alpha_1, \ldots, \alpha_m \in \N$ such that $x = y_1^{\alpha_1} \cdots y_m^{\alpha_m}$. Suppose that $\S{A}(q_t) = \cbrace{a_t}$ for all $t \in [l]$. Then there are integers $i_0, \ldots, i_l \in \N$ such that $0 = i_0 < i_1 < \cdots < i_l = m$ and 
\[
a_t = y_{i_{t-1}+1} = \cdots = y_{i_t}
\] 
for all $t \in [l]$. We set $\beta_t = i_t - (i_{t-1}+1)$ and $\gamma_t = \alpha_{i_{t-1}+1} + \cdots + \alpha_{i_t}$ for all $t \in [l]$. From $1 \leq \alpha_i$ for all $i \in [m]$ follows $\beta_t \leq \gamma_t$ for all $t \in [l]$. Hence, we have 
\begin{align*}
y = a_1^{\beta_1} \cdots a_l^{\beta_l} \quad \text{ and } \quad 
x = a_1^{\gamma_1} \cdots a_l^{\gamma_l}.
\end{align*}
Hence, $k = l$, $\S{A}(p_t) = \S{A}(q_t)$, and $\abs{p_t} = \gamma_t \geq \beta_t=\abs{q_t}$ for all $t \in [l]$. This shows $x \succ y$. 
\QED\end{proof}

\medskip

The next result shows that the set of compressions of an irreducible word is a singleton. 

\begin{proposition}
Let $x \in \S{A}^*$ be irreducible. Then $\S{C}(x) = \cbrace{x}$. 
\end{proposition}

\begin{proof}
By definition, we have $x \in \S{C}(x)$. Suppose there is a word $y \in \S{C}(x)$ with prime factorization $\Pi(y) = p_1 \cdots p_k$. Let $y^* = a_1 \cdots a_k$ be the condensed form of $y$, where $a_t \in \S{A}(p_t)$ for all $t \in [k]$. From $y \prec x$ and $y^* \prec y$ follows $y^* \prec x$ by Prop.~\ref{prop:expansions-are-transitive}. According to Lemma \ref{lemma:expansions} there are positive integers $\alpha_1, \ldots, \alpha_k \in \N$ such that $x = a_1^{\alpha_1} \cdots a_k^{\alpha_k}$. Since $x$ is irreducible all $\alpha_t$ have value one giving $x = a_1 \cdots a_k$. Hence, we have $x = y^*$. In addition, from $x \succ y$ and $y \succ x$ follows $x = y$. This shows the assertion.
\QED\end{proof}

\medskip

Proposition \ref{prop:compression-is-expansion-of-condensed-form} states that every compression of a word is an expansion of its condensed form. 

\begin{proposition}\label{prop:compression-is-expansion-of-condensed-form}
Let $x$ be a word over $\S{A}$ with condensed form $x^*$. Suppose that $y\in \S{A}^*$ such that $y \prec x$. Then $x^* \prec y$. 
\end{proposition}

\begin{proof}
Let $x \in \S{A}^*$ be a word with prime factorization $\Pi(x) = p_1 \cdots p_k$. Suppose that $y \in \S{C}(x)$ with condensed form 
$\Pi(y) = q_1 \cdots q_l$. From $y \prec x$ follows that $k = l$ and $q_t \prec p_t$ for all $t \in [k]$. This implies that $x$ and $y$ have the same condensed form $x^*$. Hence, we have $x^* \prec y$, which completes the proof.
\QED\end{proof}

\medskip

Suppose that $\S{C}(x)$ is the set of compressions of a word $x$. We show that the shortest word in $\S{C}(x)$ is the condensed form of $x$. 

\begin{proposition}\label{prop:minimum-length-of-condensed-form}
Let $x$ be a word over $\S{A}$ with condensed form $x^*$. Then $\abs{x^*} < \abs{y}$ for all $y \in \S{C}(x) \setminus \cbrace{x^*}$.
\end{proposition}

\begin{proof}
Let $x^* = a_1 \cdots a_k$ and let $y \in \S{C}(x)$. From Prop.~\ref{prop:compression-is-expansion-of-condensed-form} follows that $x^* \prec y$. Lemma \ref{lemma:expansions} gives positive integers $\alpha_1, \ldots \alpha_k \in \N$ such that $y = a_1^{\alpha_1} \cdots a_k^{\alpha_k}$. This shows that $\abs{y} = \alpha_1 + \cdots + \alpha_k \geq k = \abs{x^*}$. Suppose that $\abs{y} = k$. In this case, we have $\alpha_1 = \cdots = \alpha_k = 1$ and therefore $y = x^*$. This shows that $\abs{x^*} < \abs{y}$ for all $y \in \S{C}(x)\setminus \cbrace{x^*}$.
\QED\end{proof}

Suppose that $x, y, z$ are words over $\S{A}^*$. We say, $z$ is a \emph{common compression} of $x$ and $y$ if $z \prec x, y$. A \emph{compression-expansion} (co-ex) function is a function $f: \S{A}^* \rightarrow \S{A}^*$ such that there is a common compression of $x$ and $f(x)$. Proposition \ref{prop:composition-of-co-ex-functions} states that co-ex functions are closed under compositions. 

\begin{proposition}\label{prop:composition-of-co-ex-functions}
The composition of two co-ex functions is a co-ex function.
\end{proposition}

\begin{proof}
To prove the assertion, we repeatedly apply transitivity of expansions (Prop.~\ref{prop:expansions-are-transitive}). 
Let $x \in \S{A}^*$ and let $f = g \circ h$ be the composition of two co-ex functions $g, h: \S{A}^* \rightarrow \S{A}^*$. Then there are words $z_h$ and $z_g$ such that $z_h \prec x, h(x)$ and $z_g \prec h(x), g(h(x))$. Suppose that $z_h^*$ and $z_g^*$ are the condensed forms of $z_h$ and $z_g$, respectively. From 
\begin{align*}
z_h^* \prec z_h \prec h(x)
\quad \text{ and } \quad
z_g^* \prec z_g \prec h(x)
\end{align*}
follows $z_h^* \prec h(x)$ and $z_g^* \prec h(x)$ by the transitivity of expansions. According to Prop.~\ref{prop:existence-of-condensed-form}, the condensed form of a word is unique. Hence, we have $z_h^* = z_g^*$. We set $x^* = z_h^* = z_g^*$. Then from 
\begin{align*}
x^* = z_h^* \prec z_h \prec x
\quad \text{ and } \quad
x^* = z_g^*\prec z_g \prec g(h(x)) = f(x)
\end{align*}
follows $x^* \prec x, f(x)$ by the transitivity of expansions. This shows that $x^*$ is a common compression of $x$ and $f(x)$. Since $x$ was chosen arbitrarily, the assertion follows. 
\QED\end{proof}

\subsection{Warping Walks}\label{subsec:warping-walks}

The standard definition of the dtw-distance is inconvenient for our purposes. The recursive definition of warping paths is easy to understand and well-suited for deriving algorithmic solutions, but often less suited for a theoretical analysis. In addition, warping paths are not closed under compositions. As a more convenient definition, we introduce warping walks. Warping walks generalize warping paths by slightly relaxing the step condition. Using warping functions and matrices, this section shows that warping walks do not affect the dtw-distance. The next section shows that warping walks are closed under compositions.

\medskip

\commentout{
\begin{notation}
Let $\B = \cbrace{0,1}$ and let $I_n \in \R^{n \times n}$ be the identity matrix. 
Let $A = (a_{ij}) \in \R^{m \times n}$ be a matrix. Then we write $A_i = (a_{i1}, \ldots, a_{in})$ for the $i$-th row and $A^j = (a_{1j}, \ldots, a_{mj})$ for the $j$-th column of $A$.
\qed
\end{notation}
}

\begin{notation}
Let $\B = \cbrace{0,1}$ and let $I_n \in \R^{n \times n}$ be the identity matrix. 
\qed
\end{notation}

\medskip

Let $\ell, n \in \N$. A function $\phi:[\ell] \rightarrow [n]$ is a \emph{warping function} if it is surjective and monotonically increasing. Thus, for a warping function we always have $\ell \geq n$. The \emph{warping matrix} associated with warping function $\phi$ is a matrix of the form
\[
\Phi = \begin{pmatrix}
e_{\phi(1)}\\
\vdots\\
e_{\phi(\ell)}
\end{pmatrix} \in \B^{\ell \times n},
\]
where $e_i$ is the $i$-th standard basis vector of $\R^n$, denoted as a row vector, with $1$ in the $i$-th position and $0$ in every other position. The next result shows the effect of multiplying a time series with a warping matrix. 

\begin{lemma}\label{lemma:expansion-by-phi}
Let $\phi:[\ell] \rightarrow [n]$ be a warping function with associated warping matrix $\Phi$. Suppose that $x = (x_1, \ldots, x_n)s \in \S{T}$ is a time series of length $\abs{x} = n$. Then there are elements $\alpha_1, \ldots, \alpha_n \in \N$ such that 
\[
\Phi x = (\underbrace{x_1, \ldots, x_1}_{\alpha_1-\text{times}}, \underbrace{x_2, \ldots, x_2}_{\alpha_2-\text{times}},\ldots \underbrace{x_n, \ldots, x_n}_{\alpha_n-\text{times}})^\intercal.
\] 
\end{lemma}

\begin{proof}
Since $\phi$ is surjective and monotonic, we can find integers $\alpha_1, \ldots, \alpha_n \in \N$ such that 
\[
\phi(i) = \begin{cases}
1 & i \leq \alpha_1\\
2 & \alpha_1 < i \leq \alpha_1 + \alpha_2\\
\cdots & \cdots \\
n & \alpha_1 + \cdots + \alpha_{n-1}< i 
\end{cases}
\]
for all $i \in [\ell]$. Let $\Phi \in \B^{\ell \times n}$ be the warping matrix associated with $\phi$. Then the $n$ rows $\Phi_i$ of $\Phi$ are of the form
\begin{align*}
\Phi_1 &= \cdots = \Phi_{\alpha_1} = e_1\\
\Phi_{\alpha_1+1} &= \cdots = \Phi_{\alpha_2} = e_2\\
&\;\;\vdots \\
\Phi_{\alpha_{m-1}+1} &= \cdots = \Phi_{\alpha_m} = e_m,
\end{align*}
Obviously, the warping matrix $\Phi$ satisfies 
\[
\Phi x = (\underbrace{x_1, \ldots, x_1}_{\alpha_1-\text{times}}, \underbrace{x_2, \ldots, x_2}_{\alpha_2-\text{times}},\ldots \underbrace{x_n, \ldots, x_n}_{\alpha_n-\text{times}})^\intercal.
\] 
\QED\end{proof}

\medskip

A \emph{warping walk} is a pair $w = (\phi, \psi)$ consisting of warping functions $\phi:[\ell] \rightarrow [m]$ and $\psi:[\ell] \rightarrow [n]$ of the same domain $[l]$. The warping walk $w$ has \emph{order} $m \times n$ and \emph{length} $\ell$. By $\S{W}_{m,n}$ we denote the set of all warping walks of order $m \times n$ and of finite length.

In the classical terminology of dynamic time warping, a warping walk can be equivalently expressed by a sequence $w = (w_1, \ldots, w_{\ell})$ of $\ell$ points $w_l = (\phi(l), \psi(l)) \in [m] \times [n]$ such that the following conditions are satisfied:
\begin{enumerate}
\item $w_1 = (1,1)$ and $w_\ell = (m,n)$ \hfill (\emph{boundary condition})
\item $w_{l+1} - w_{l} \in \B \times \B$ for all $l \in [\ell-1]$ \hfill(\emph{weak step condition})
\end{enumerate}
The weak step condition relaxes the standard step condition of warping paths by additionally allowing zero-steps of the form $w_l - w_{l+1} = (0,0)$. Zero-steps duplicate points $w_l$ and thereby admit multiple correspondences between the same elements of the underlying time series. 

\begin{notation}
We identify warping walks $(\phi, \psi)$ with their associated warping matrices $(\Phi, \Psi)$ and interchangeably write $(\phi, \psi) \in \S{W}_{m,n}$ and $(\Phi, \Psi) \in \S{W}_{m,n}$.
\end{notation}

\medskip

A warping walk $p = (p_1, \ldots, p_{\ell})$ is a \emph{warping path} if $p_{l+1} \neq p_l$ for all $l \in [\ell-1]$. By $\S{P}_{m,n}$ we denote the subset of all warping paths of order $m \times n$. Thus, a warping path is a warping walk without consecutive duplicates. Equivalently, a warping path satisfies the boundary condition and the strict step condition
\begin{enumerate}
\item[$2'.$] $w_{l+1} - w_{l} \in \B \times \B \setminus \cbrace{0,0}$ for all $l \in [\ell-1]$ \hfill(\emph{strict step condition})
\end{enumerate}

Warping walks are words over the alphabet $\S{A} = \N \times \N$ and warping paths are irreducible words over $\S{A}$. For the sake of convenience, we regard $\S{W}_{m,n}$ and $\S{P}_{m,n}$ as subsets of $\S{A}^*$. The \emph{condensation map}
\[
c: \S{A}^* \rightarrow \S{A}^*, \quad w \mapsto w^*
\]
sends a word $w$ over $\S{A}$ to its condensed form $w^*$. 
\begin{lemma}\label{lemma:c(W)=P}
Let $\S{A} = \N \times \N$ and let $c: \S{A}^* \rightarrow \S{A}^*$ be the condensation map. Then $c(\S{W}_{m,n}) = \S{P}_{m,n}$ for all $m,n \in \N$. 
\end{lemma}

\begin{proof}
Let $m,n \in \N$. 

\medskip

\noindent
$\S{P}_{m,n}\subseteq c(\S{W}_{m,n})$: Let $p = (p_1, \ldots, p_\ell) \in \S{P}_{m,n}$ be a warping path. From the strict step condition follows that $p$ is irreducible. Consider the word $w = (p_1, \ldots, p_\ell, p_{\ell+1})$, where $p_\ell = p_{\ell+1}$. The word $w$ satisfies the boundary and weak step condition. Hence, $w$ is a warping walk with unique prime factorization $\Pi(w) = (p_1 \cdots p_\ell)$. This shows that $p$ is the unique condensed form of $w$. Hence, we have $\S{P}_{m,n}\subseteq c(\S{W}_{m,n})$.

\medskip

\noindent
$c(\S{W}_{m,n})\subseteq \S{P}_{m,n}$: A warping walk $w \in \S{W}_{m,n}$ satisfies the boundary and the weak step condition. As an irreducible word, the condensed form $w^* = c(w)$ satisfies the boundary and the strict step condition. Hence, $w^*$ is a warping path. This proves $c(\S{W}_{m,n})\subseteq \S{P}_{m,n}$.
\QED\end{proof}

\medskip

The \emph{dtw-distance} is a distance function on $\S{T}$ of the form
\[
\dtw: \S{T} \times \S{T} \rightarrow \R_{\geq 0}, \quad (x, y) \mapsto \min \cbrace{\norm{\Phi x - \Psi y} \,:\, (\Phi, \Psi) \in \S{P}_{\abs{x},\abs{y}}}
\]
From \cite{Schultz2018}, Prop.~A.2 follows that the dtw-distance coincides with the standard definition of the dtw-distance as presented in Section \ref{sec:results}. The next result expresses the dtw-distance in terms of warping walks.

\begin{proposition}\label{prop:dtw}
Let $x, y \in \S{T}$ be two time series. Then we have
\begin{align*}
\dtw(x, y) = \min \cbrace{\norm{\Phi x - \Psi y} \,:\, (\Phi, \Psi) \in \S{W}_{\abs{x},\abs{y}}}.
\end{align*}
\end{proposition}

\begin{proof}
Let $w \in \S{W}_{m,n}$ be a warping walk. Then $p=c(w)$ is a warping path and a condensed form of $w$ by Lemma \ref{lemma:c(W)=P}. From Prop.~\ref{prop:minimum-length-of-condensed-form} follows that $\abs{p} \leq \abs{w}$. Thus, we obtain
\[
C_p(x, y) = \sum_{(i,j) \in p} (x_i-y_j)^2 \leq \sum_{(i,j) \in w} (x_i-y_j)^2 = C_w(x,y).
\] 
This implies
\[
\dtw(x, y) \leq \min \cbrace{\norm{\Phi x - \Psi y} \,:\, (\Phi, \Psi) \in \S{W}_{m,n} \setminus \S{P}_{m,n}}
\]
and proves the assertion.
\QED\end{proof}

\medskip

We call a warping walk $(\Phi, \Psi)$ \emph{optimal} if $\norm{\Phi x - \Psi y} = \dtw(x, y)$. From Prop.~\ref{prop:dtw} follows that transition from warping paths to warping walks leaves the dtw-distance unaltered and that we can condense every optimal warping walk to an optimal warping path by removing consecutive duplicates.

\subsection{Properties of Warping Functions}\label{subsec:properties-of-warping-walks}

In this section, we compile results on compositions of warping functions and warping walks. We begin with showing that warping functions are closed under compositions. 
\begin{lemma}\label{lemma:composition-01}
Let $\phi:[\ell] \rightarrow [m]$ and $\psi:[m] \rightarrow [n]$ be warping functions. Then the composition
\[
\theta: [\ell] \rightarrow [n], \quad l \mapsto \psi(\phi(l))
\]
is also a warping function. 
\end{lemma}

\begin{proof}
The assertion follows, because the composition of surjective (monotonic) functions is surjective (monotonic). 
\QED\end{proof}

\medskip

The composition of warping functions is contravariant to the composition of their associated warping matrices. Suppose that $\Phi \in \B^{\ell \times m}$ and $\Psi \in \B^{m \times n}$ are the warping matrices of the warping functions $\phi$ and $\psi$ from Lemma \ref{lemma:composition-01}, respectively. Then the warping matrix of the composition $\theta = \psi \circ \phi$ is the matrix product $\Phi\Psi \in \B^{\ell \times n}$. The next result shows that warping walks are closed for a special form of compositions. 

\begin{lemma}\label{lemma:composition-02}
Let $(\phi, \psi) \in \S{W}_{m,n}$ be a warping walk and let $\theta: [m] \rightarrow [r]$ be a warping function. Then $(\theta \circ \phi, \psi)$ is a warping walk in $\S{W}_{r,n}$. 
\end{lemma}

\begin{proof}
Follows from Lemma \ref{lemma:composition-01} and by definition of a warping walk.
\QED\end{proof}

\medskip

Let $\phi:[m] \rightarrow [n]$ and $\phi':[m'] \rightarrow [n]$ be warping functions. The \emph{pullback} of $\phi$ and $\phi'$ is the set of the form
\[
\phi \otimes \phi' = \cbrace{(u,u') \in [m] \times [m'] \,:\, \phi(u) = \phi'(u')}.
\]
By $\pi: \phi \otimes \phi' \rightarrow [m]$ and $\pi': \phi \otimes \phi' \rightarrow [m']$ we denote the canonical projections. Let 
$\psi = \phi \circ \pi$ and $\psi' = \phi' \circ \pi'$ be the compositions that send elements from the pullback $\phi \otimes \phi'$ to the set $[n]$. The \emph{fiber} of $i \in [n]$ under the map $\psi$ is the set $\S{F}(i) = \cbrace{(u,u') \in \phi \otimes \phi' \,:\, \psi(u,u') = i}$. In a similar way, we can define the fiber of $i$ under the map $\psi'$. The next results show some properties of pullbacks and their fibers. 

\begin{lemma}\label{lemma:psi=psi'}
Let $\phi \otimes \phi'$ be a pullback of warping functions $\phi$ and $\phi'$. Then the compositions $\psi = \phi \circ \pi$ and $\psi' = \phi' \circ \pi'$ are surjective and satisfy $\psi = \psi'$.
\end{lemma}

\begin{proof}
Warping functions and the natural projections are surjective. As a composition of surjective functions, the functions $\psi$ and $\psi'$ are surjective. For every $(u,u') \in \phi \otimes \phi'$ we have
\[
\psi(u,u') = \phi(\pi(u,u')) = \phi(u) = \phi'(u') = \phi'(\pi'(u,u')) = \psi'(u,u').
\]
This proves the assertion $\psi = \psi'$. 
\QED\end{proof}

\medskip

Lemma \ref{lemma:psi=psi'} has the following implications: First, from $\psi = \psi'$ follows that the fiber of $i$ under the map $\psi$ coincides with the fiber of $i$ under the map $\psi'$. Second, since $\psi$ is surjective, the fibers $\S{F}(i)$ are non-empty for every $i \in [n]$. Third, the fibers $\S{F}(i)$ form a partition of the pullback $\phi \otimes \phi'$.

\begin{lemma}\label{lemma:fibers-are-order-preserving}
Let $\phi \otimes \phi'$ be a pullback of warping functions $\phi:[m] \rightarrow [n]$ and $\phi':[m'] \rightarrow [n]$. Suppose that $i, j \in [n]$ with $i < j$. From $(u_i, u'_i) \in \S{F}(i)$ and $(u_j, u'_j) \in \S{F}(j)$ follows $u_i \leq u_j$ and $u'_i \leq u'_j$.
\end{lemma}
\begin{proof}
We show the first assertion $u_i \leq u_j$. The proof for the second assertion $u'_i \leq u'_j$ is analogous. Suppose that $u_i > u_j$. From $(u_i, u'_i) \in \S{F}(i)$ follows $\phi(u_i) = i$ and from $(u_j, u'_j) \in \S{F}(j)$ follows $\phi(u_j) = j$. Since $\phi$ is monotonic, we have $i = \phi(u_i) \geq \phi(u_j) = j$, which contradicts the assumption that $i < j$. This shows $u_i \leq u_j$. 
\QED\end{proof}

\begin{lemma}\label{lemma:representation-of-fibers}
Let $\phi \otimes \phi'$ be a pullback of warping functions $\phi:[m] \rightarrow [n]$ and $\phi':[m'] \rightarrow [n]$. For every $i \in [n]$ there are elements $u_i \in [m]$, $u'_i \in [m']$ and $k_i, l_i \in \N$ such that 
\begin{align*}
\pi(\S{F}(i)) = \cbrace{u_i, u_i + 1, \ldots, u_i + k_i} \quad \text{ and } \quad
\pi'(\S{F}(i)) = \cbrace{u'_i, u'_i + 1, \ldots, u'_i + l_i}.
\end{align*}
\end{lemma}
\begin{proof}
We show the assertion for $\pi(\S{F}(i))$. The proof of the assertion for $\pi'(\S{F}(i))$ is analogous. Let $i \in [n]$ and let $\S{G}(i) = \pi(\S{F}(i))$. Since fibers are non-empty and finite, we can find elements $u_i \in [m]$ and $k_i \in \N$ such that 
$\S{G}(i) = \cbrace{u_i, u_{i+1}, \ldots, u_{i + k_i}}$ with $u_i < u_{i+1} < \cdots < u_{i + k_i}$. It remains to show that $u_{i+r} = u_i + r$ for all $r \in [k_i]$.

We assume that there is a smallest number $r \in [k_i]$ such that $u_{i+r} \neq u_i + r$. Then $r \geq 1$ and therefore $i+r-1 \geq i$. This shows that $u_{i+r-1} \in \S{G}(i)$. Observe that $u_{i+r-1} = u_i + r-1$, because $r$ is the smallest number violating $u_{i+r} = u_i + r$. From $u_{i+r-1} < u_{i+r}$ together with $u_i + r \notin \S{G}(i)$ follows
\[
u_{i+r-1} = u_i + r-1 < u_i + r < u_{i+r}.
\]
Recall that the fibers form a partition of the pullback $\phi \otimes \phi'$. Then there is a $j \in [n] \setminus \cbrace{i}$ such that $u_i+r \in \S{G}(j)$. We distinguish between two cases:\footnote{The case $i = j$ can not occur by assumption.} 

\medskip

\noindent
Case $i < j$: From Lemma \ref{lemma:fibers-are-order-preserving} follows that $u_{i+r} \leq u_i+r$, which contradicts the previously derived inequality $u_i + r < u_{i+r}$.

\medskip

\noindent
Case $j < i$: From Lemma \ref{lemma:fibers-are-order-preserving} follows that $u_i+r \leq u_i$. Observe that either $u_i = u_{i+r-1}$ or $u_i < u_{i+r-1}$. We obtain the contradiction $u_i \leq u_{i+r-1} < u_i + r \leq u_i$. 

\medskip

From both contradictions follows that $u_{i+r} = u_i + r$ for every $r \in [k_i]$. This completes the proof. 
\QED\end{proof}

\medskip

Lemma \ref{lemma:pullback} uses pullbacks to show that pairs of warping functions with the same co-domain can be equalized by composition with a suitable warping walk. 

\begin{lemma}\label{lemma:pullback}
Let $\phi:[m] \rightarrow [n]$ and $\phi':[m'] \rightarrow [n]$ be warping functions. Then there are warping functions $\theta: [\ell] \rightarrow [m]$ and $\theta': [\ell] \rightarrow [m']$ for some $\ell \geq \max(m, m')$ such that $\phi \circ \theta = \phi' \circ \theta'$. 
\end{lemma}

\begin{proof}
We first construct a suitable set $\S{Z}$ of cardinality $\abs{Z} = \ell$. For this, let $\phi \otimes \phi'$ be the pullback of $\phi$ and $\phi'$ and let $i \in [n]$. From Lemma \ref{lemma:representation-of-fibers} follows that there are elements $u_i \in [m]$, $u'_i \in [m']$ and $k_i, l_i \in \N$ such that 
\begin{align*}
\pi(\S{F}(i)) = \cbrace{u_i, u_{i + 1}, \ldots, u_{i + k_i}} 
\qquad \text{and} \qquad
\pi'(\S{F}(i)) = \cbrace{u'_i, u'_{i + 1}, \ldots, u'_{i + l_i}}.
\end{align*}
where $u_{i+r} = u_i + r$ for all $r \in [k_i]$ and $u'_{i+s} = u'_i + s$ for all $s \in [l_i]$. Without loss of generality we assume that $k_i \leq l_i$. For every $i \in [n]$, we construct the subset
\[
\S{Z}(i) = \cbrace{(u_i, u'_i), (u_{i+1}, u'_{i+1}), \ldots, (u_{i + k_i}, u'_{i+k_i}), (u_{i + k_i}, u'_{i+k_i+1})\ldots, (u_{i + k_i}, u'_{i+l_i})} \subseteq \S{F}(i)
\]
and form their disjoint union
\[
\S{Z} = \bigcup_{i \in [n]} \S{Z}(i) \subseteq \phi \otimes \phi'.
\]
Let $\leq_{\S{Z}}$ be the lexicographical order on $\S{Z}$ defined by
\[
(u_i,u'_i) \leq_{\S{Z}} (u_j,u'_j) \quad \text{ if and only if } \quad (u_i < u_j) \text{ or } (u_i = u_j \text{ and } u'_i \leq u'_j).
\]
for all $(u_i,u'_i), (u_j,u'_j) \in \S{Z}$. We show that the properties of $\S{Z}$ yield a tighter condition on $\leq_{\S{Z}}$. Let $(u_i,u'_i), (u_j,u'_j) \in \S{Z}$ such that $(u_i,u'_i) \leq_{\S{Z}} (u_j,u'_j)$. Then there are $i,j \in [n]$ such that $(u_i,u'_i) \in \S{Z}(i)$ and $(u_j,u'_j) \in \S{Z}(j)$. We distinguish between three cases:
\begin{enumerate}
\itemsep0em
\item $i = j$: By construction of $\S{Z}(i)$, the relationship $u_i \leq u_j$ gives $u'_i \leq u'_j$.
\item $i < j$: From Lemma \ref{lemma:fibers-are-order-preserving} follows that $u_i \leq u_j$ and $u'_i \leq u'_j$. 
\item $i > j$: Lemma \ref{lemma:fibers-are-order-preserving} yields $u_i \geq u_j$ and $u'_i \geq u'_j$. The assumption $(u_i,u'_i) \leq_{\S{Z}} (u_j,u'_j)$ gives $u_i \leq u_j$. Then from $u_i \geq u_j$ and $u_i \leq u_j$ follows $u_i = u_j$. In addition, we have $u'_i \leq u'_j$ by $(u_i,u'_i) \leq_{\S{Z}} (u_j,u'_j)$ and $u_i = u_j$. Hence, from $u'_i \geq u'_j$ and $u'_i \leq u'_j$ follows $u'_i = u'_j$. In summary, we have $u_i = u_j$ and $u'_i = u'_j$.
\end{enumerate}
The case distinction yields 
\[
(u_i,u'_i) \leq_{\S{Z}} (u_j,u'_j) \quad \text{ if and only if } \quad (u_i \leq u_j) \text{ and } (u'_i \leq u'_j).
\]
for all $(u_i,u'_i), (u_j,u'_j) \in \S{Z}$. Suppose that $\ell = \abs{\S{Z}}$. Then there is a bijective function
\[
f: [\ell] \rightarrow \S{Z}, \quad i \mapsto f(i)
\]
such that $i \leq j$ implies $f(i) \leq_{\S{Z}} f(j)$ for all $i,j \in [\ell]$. 

Next, we show that the functions $\theta = \pi \circ f$ and $\theta' = \pi' \circ f$ are warping functions on $[\ell]$. Both functions $\theta$ and $\theta'$ are surjective as compositions of surjective functions. To show that $\theta$ and $\theta'$ are monotonic, we assume that $i, j \in [\ell]$ with $i < j$. Suppose that $f(i) = (u_i,u'_i)$ and $f(j) = (u_j,u'_j)$. From $i < j$ follows $(u_i,u'_i) \leq_{\S{Z}} (u_j,u'_j)$ by construction of $f$. From $(u_i,u'_i) \leq_{\S{Z}} (u_j,u'_j)$ follows $u_i \leq u_j$ and $u'_i \leq u'_j$ as shown in the first part of this proof. Hence, we find that 
\begin{align*}
\theta(i) = \pi(f(i)) = \pi(u_i, u'_i) = u_i &\leq u_j = \pi(u_j, u'_j) = \pi(f(j)) = \theta(j)\\
\theta'(i) = \pi'(f(i)) = \pi'(u_i, u'_i) = u'_i &\leq u'_j = \pi'(u_j, u'_j) = \pi'(f(j)) = \theta'(j).
\end{align*}
Thus, $\theta$ and $\theta'$ are monotonic. This proves that $\theta$ and $\theta'$ are warping functions. 

It remains to show $\phi \circ \theta = \phi' \circ \theta'$. From Lemma \ref{lemma:psi=psi'} follows $\phi \circ \pi = \phi' \circ \pi'$. Then we have 
\[
\phi \circ \theta = (\phi \circ \pi) \circ f = (\phi' \circ \pi') \circ f = \phi' \circ \theta'.
\]
This completes the proof.
\QED\end{proof}

\subsection{Semi-Metrification of DTW-Spaces}\label{subsec:semi-metrification}

In this section, we convert the dtw-distance to a semi-metric. For this, we regard time series as words over the alphabet $\S{A} = \R$. The set of finite time series is denoted by $\S{T} = \S{A}^*$. The next result shows that expansions of words over numbers can be expressed by matrix multiplication. 

\begin{lemma}\label{lemma:x>y=>x=PHIy}
Let $x, y \in \S{T}$ be two time series. Then the following statements are equivalent:
\begin{enumerate}
\item $x$ is an expansion of $y$.
\item There is a warping matrix $\Phi$ such that $x = \Phi y$.
\end{enumerate}
\end{lemma}

\begin{proof}
Suppose that $\abs{x} = n$ and $y = (y_1, \ldots, y_m)$. 

\medskip

\noindent
$\Rightarrow$: We assume that $x \succ y$. According to Lemma \ref{lemma:expansions} there are positive integers $\alpha_1, \ldots, \alpha_m \in \N$ such that $n = \alpha_1 + \cdots + \alpha_m$ and $x = y_1^{\alpha_1} \cdots y_m^{\alpha_m}$. Consider the function $\phi:[n] \rightarrow [m]$ defined by
\[
\phi(i) = \begin{cases}
1 & i \leq \alpha_1\\
2 & \alpha_1 < i \leq \alpha_1 + \alpha_2\\
\cdots & \cdots \\
m & \alpha_1 + \cdots \alpha_{m-1}< i 
\end{cases}
\]
for all $i \in [n]$. The function $\phi$ is surjective: Suppose that $j \in [m]$. We set $i = \alpha_1 + \cdots + \alpha_j$. Then $1 \leq i \leq n$ and $\phi(i) = j$ by definition of $\phi$. By construction, the function $\phi$ is also monotonically increasing. Hence, $\phi$ is a warping function. Then from Lemma \ref{lemma:expansion-by-phi} follows the second statement. 

\medskip

\noindent
$\Leftarrow$: Let $\Phi \in \B^{n \times m}$ be a warping matrix such that $x = \Phi y$. Then there is a warping function $\phi: [n] \rightarrow [m]$ associated with $\Phi$. The first statement follows by first applying Lemma \ref{lemma:expansion-by-phi} and then by Lemma \ref{lemma:expansions}.
\QED\end{proof}

\medskip

\emph{Warping identification} is a relation on $\S{T}$ defined by $x \sim y \,\Leftrightarrow\, \dtw(x, y) = 0$ for all $x, y \in \S{T}$. We show that warping identification is an equivalence relation.

\begin{proposition}\label{prop:warping-identification-class}
The warping-identification $\sim$ is an equivalence relation on $\S{T}$. 
\end{proposition}

\begin{proof}
The relation $\sim$ is reflexive and symmetric by the properties of the dtw-distance. It remains to show that the warping-identification is transitive. Let $x, y, z \in \S{T}$ be time series with $x \sim y$ and $y \sim z$. Then $\dtw(x, y) = \dtw(y, z) = 0$. Hence, there are optimal warping paths $(\Phi, \Psi)$ and $(\Phi', \Psi')$ of length $\ell$ and $\ell'$, resp., such that $\norm{\Phi x - \Psi y} = \norm{\Phi'y - \Psi' z} = 0$. From Lemma \ref{lemma:pullback} follows that there are warping matrices $\Theta$ and $\Theta'$ of the same length $\ell$ such that $\Theta\Psi y = \Theta'\Psi'y$. Observe that
\begin{align*}
\norm{\Theta\Phi x - \Theta' \Psi' z} 
&= \norm{\Theta\Phi x - \Theta\Psi y + \Theta'\Phi' y - \Theta' \Psi' z}\\
&\leq \norm{\Theta\Phi x - \Theta\Psi y} + \norm{\Theta'\Phi' y - \Theta' \Psi' z}\\
&\leq \norm{\Theta}\norm{\Phi x - \Psi y} + \norm{\Theta'}\norm{\Phi'y - \Psi' z}\\
&= 0.
\end{align*}
Note that the second inequality in the third line follows from the fact that the Frobenius norm on matrices is compatible to the vector norm. From Lemma \ref{lemma:composition-01} follows that the products $\Theta\Phi$ and $\Theta' \Psi'$ are warping matrices. Since both products have the same length $\ell$, we find that the pair $(\Theta\Phi, \Theta' \Psi')$ is a warping walk. Then from Prop.~\ref{prop:dtw} follows that $\dtw(x, z) \leq \norm{\Theta\Phi x - \Theta' \Psi' z} = 0$. This proves that $\sim$ is transitive and completes the proof.
\QED\end{proof}

For every $x \in \S{T}$ let $[x] = \cbrace{y \in \S{T} \,:\, x \sim y}$ denote the \emph{warping-identification class} of $x$. The next result presents an equivalent definition of the warping-identification class. 

\begin{proposition}\label{prop:generator-of-[x]}
Let $x \in \S{T}$ be a time series with condensed form $x^*$. Then the warping-identification class of $x$ is of the form
\[
[x] = \cbrace{y \in \S{T} \,:\, y \succ x^*}.
\] 
\end{proposition}

\begin{proof}
The warping-identification class is defined by 
\[
[x] = \cbrace{y \in \S{T} \,:\, x \sim y} = \cbrace{y \in \S{T} \,:\, \dtw(x,y) = 0}.
\] 
Let $\S{E}(x^*) = \cbrace{y \in \S{T} \,:\, y \succ x^*}$ denote the set of expansions of $x^*$. We show that $[x] = \S{E}(x^*)$.

\medskip

\noindent
$\subseteq$: Let $y \in [x]$. There is an optimal warping path $(\Phi, \Psi)$ such that $\dtw(x, y) = \norm{\Phi x - \Psi y} = 0$. From Lemma \ref{lemma:x>y=>x=PHIy} follows that there is a warping matrix $\Theta$ with $x = \Theta x^*$. Hence, $\norm{\Phi\Theta x^* - \Psi y} = 0$ and we obtain $\Phi\Theta x^* = \Psi y$. From Lemma \ref{lemma:composition-01} follows that the product $\Phi\Theta$ of warping matrices $\Phi$ and $\Theta$ is a warping matrix. We set $z = \Phi\Theta x^* = \Psi y$. Then $z \succ x^*$ and $z \succ y$ by Lemma \ref{lemma:x>y=>x=PHIy}. From Prop.~\ref{prop:compression-is-expansion-of-condensed-form} follows that $y \succ x^*$. This shows that $y \in \S{E}(x^*)$.

\medskip

\noindent
$\supseteq$: Let $y \in \S{E}(x^*)$. We assume that $\abs{x} = n$, $\abs{y}=m$, and $\abs{x^*} = k$. From Lemma \ref{lemma:x>y=>x=PHIy} follows that there are warping matrices $\Phi \in \B^{n \times k}$ and $\Psi \in \B^{m \times k}$ with $x = \Phi x^*$ and $y = \Psi x^*$, respectively. Then $(\Phi, I_n)$ and $(\Psi, I_{m})$ are warping walks of length $n$ and $m$ respectively. We have
\begin{align*}
0 \leq \dtw(x, x^*) &\leq \norm{I_n x - \Phi x^*} = 0 \\
0 \leq \dtw(y, x^*) &\leq \norm{I_{m} y - \Psi x^*} = 0
\end{align*}
giving $\dtw(x, x^*) = \dtw(y, x^*) = 0$. Hence, we have $x \sim x^*$ and $y \sim x^*$. Since $\sim$ is an equivalence relation, we have $x \sim y$ by Prop.~\ref{prop:warping-identification-class}. This proves $y \in [x]$. 
\QED\end{proof}

\medskip

Proposition \ref{prop:expansion-inequality} states that expansions do not decrease the dtw-distance to other time series.

\begin{proposition}\label{prop:expansion-inequality} 
Let $x,y \in \S{T}$ be time series such that $x \succ y$. Then 
\begin{enumerate}
\itemsep0em
\item $\dtw(x, y) = 0$ 
\item $\dtw(x, z) \geq \dtw(y, z)$ for all $z \in \S{T}$. 
\end{enumerate}
\end{proposition}

\begin{proof}
We first show the second assertion. Let $z \in \S{T}$ be a time series and let $(\Phi, \Psi)$ be an optimal warping path between $x$ and $z$. Then we have $\delta(x, z) = \norm{\Phi x - \Psi z}$ by Prop.~\ref{prop:dtw}. From $x \succ y$ and Lemma \ref{lemma:x>y=>x=PHIy} follows that there is a warping matrix $\Theta$ such that $x = \Theta y$. We obtain
\begin{align*}
\delta(x, z) = \norm{\Phi x - \Psi z} = \norm{\Phi \Theta y - \Psi z} \geq \delta(y, z).
\end{align*}
From Lemma \ref{lemma:composition-02} follows that $(\Phi\Theta, \Psi)$ is a warping walk. The inequality holds, because $(\Phi\Theta, \Psi)$ is not necessarily an optimal warping walk between $y$ and $z$. 

The first assertion follows from the second one by setting $z = x$. We obtain
\[
0 = \dtw(x, x) \geq \dtw(y, x) \geq 0. 
\]
This implies $\dtw(y, x) = 0$ and completes the proof.
\QED\end{proof}

The set $\S{T}^* = \cbrace{[x] \,:\, x \in \S{T}}$ is the quotient space of $\S{T}$ under warping identification $\sim$. We define the distance function 
\[
\delta^*: \S{T}^* \times \S{T}^* \rightarrow \R_{\geq 0}, \quad ([x], [y]) \mapsto \inf_{x' \in [x]}\; \inf_{y' \in [y]}\; \delta(x', y').
\]
We call $\delta^*$ the \emph{quotient distance} induced by $\delta$. 

\begin{theorem}\label{theorem:semi-metric}
The quotient distance $\delta^*$ induced by the dtw-distance $\delta$ is a well-defined semi-metric satisfying $\delta^*([x], [y]) = \delta(x^*,y^*)$ for all $x, y \in \S{T}$. 
\end{theorem}

\begin{proof}
Let $x^*$ and $y^*$ be the condensed forms of $x$ and $y$, respectively. Repeatedly applying Prop.~\ref{prop:expansion-inequality} gives 
\begin{align*}
\dtw(x^*,y^*) \leq \dtw(x^*,y') \leq \dtw(x',y')
\end{align*}
for all $x' \in [x]$ and all $y' \in [y]$. Hence, the infimum exists and $\dtw^*([x], [y]) = \dtw(x^*, y^*)$ is well-defined. 

We show that $\delta^*$ is a semi-metric. Non-negativity and symmetry of $\dtw^*$ follow from non-negativity and symmetry of $\dtw$. To show the identity of indiscernibles, we assume that $\dtw^*([x], [y]) = 0$. From the identity $\delta^*([x], [y]) = \delta(x^*,y^*)$ follows $\dtw(x^*, y^*) = 0$. This implies that $x^*$ and $y^*$ are warping identical, that is $x^* \sim y^*$. By Prop.~\ref{prop:warping-identification-class} we have the following relations $[x^*] = [y^*]$, $[x] = [x^*]$, and $[y] = [y^*]$. Combining these relations gives $[x] = [y]$. This shows that $\dtw^*$ satisfies the identity of indiscernibles. Hence, $\delta^*$ is a semi-metric.
\QED\end{proof}

\subsection{Warping-Invariance}\label{subsec:warping-invariance}

This section shows that the canonical extension of the quotient distance is warping-invariant. A distance function $d:\S{T} \times \S{T} \rightarrow \R_{\geq 0}$ is \emph{warping-invariant} if 
\[
d(x, y) = d(x', y')
\]
for all time series $x, y, x', y' \in \S{T}$ with $x \prec x'$ and $y \prec y'$. The quotient distance $\delta^*$ extends to a distance on $\S{T}$ by virtue of
\[
\delta^\sim: \S{T} \times \S{T} \rightarrow \R_{\geq 0}, \quad (x, y) \mapsto \delta^*([x], [y]).
\]
We call $\dtw^\sim$ the \emph{canonical extension} of $\dtw^*$.

\begin{theorem}\label{theorem:warping-invariance} The canonical extension $\dtw^\sim$ of the quotient distance $\dtw^*$ is warping-invariant.
\end{theorem}

\begin{proof}
Let $x, x', y, y' \in \S{T}$ be time series such that there are common compressions $u \prec x, x'$ and $v \prec y, y'$. We show that $\delta^\sim(x, y) = \delta^\sim(x', y')$. Suppose that $x^*$ and $y^*$ are the condensed forms of $u$ and $v$, respectively. By repeatedly applying Prop.~\ref{prop:expansions-are-transitive} we obtain $x, x' \succ x^*$ and $y, y' \succ y^*$. From Prop.~\ref{prop:generator-of-[x]} follows that $[x] = [x']$ and $[y] = [y']$. This gives $\delta^\sim(x, y) = \delta^*([x], [y]) = \delta^*([x'], [y']) = \delta^\sim(x', y')$. The proof is complete. 
\QED\end{proof}

\commentout{
\begin{proposition}\label{prop:equivalent-formulation-warping-invariance}
Let $d:\S{T} \times \S{T} \rightarrow \R_{\geq 0}$ be a distance function. Then the following statements are equivalent:
\begin{enumerate}
\item $d$ is warping-invariant.
\item $d(x, y) = d(x', y')$ for all $x, y, x', y' \in \S{T}$ with $x \sim x'$ and $y \sim y'$. 
\end{enumerate}
\end{proposition}
\begin{proof}
We first assume that $d$ is warping-invariant. Let $x, y, x', y' \in \S{T}$ be time series such that $x \sim x'$ and $y \sim y'$. Proposition~\ref{prop:warping-identification-class} gives $[x] = [x']$ and $[y] = [y']$. From Prop.~\ref{prop:generator-of-[x]} follows that there are condensed forms $x^*$ and $y^*$ that generate the equivalence classes $[x] = [x']$ and $[y] = [y']$, respectively. Thus we have $x^* \prec x, x'$ and $y^* \prec y, y'$. Since $d$ is warping-invariant, we have 
\[
d(x^*, y^*) = d(x, y) \quad \text{ and } \quad d(x^*, y^*) = d(x', y').
\]
Combining both equations yields $d(x, y) = d(x', y')$. This shows that the first statement implies the second.

\medskip

We assume that the second statement holds. Let $x, y, x', y' \in \S{T}$ be time series such that $x \prec x'$ and $y \prec y'$. From Prop.~\ref{prop:expansion-inequality} follows that $\dtw(x, x') = 0$ and $\dtw(y, y') = 0$. Hence, we find that $x \sim x'$ and $y \sim y'$. This implies $d(x, y) = d(x', y')$ by assumption. Hence, the first statement holds. 
\QED\end{proof}
}

\section{Experiments}\label{sec:experiments}

The goal of these experiments is (i) to assess the prevalence of reducible (non-irreducible) time series and (ii) to assess the performance of the nearest-neighbor classifier in the semi-metric quotient space $\args{\S{T}^*, \delta^*}$.

\subsection{Dataset}

We used $85$ datasets of the UEA \& UCR Time Series Classification Repository \cite{Bagnall2018}. Every dataset consists of time series of identical length and comes with a predefined partition into a training and test set. Table \ref{tab:condensation} shows the datasets along with the length and number of time series.

\subsection{Prevalence of Reducible Time Series}

The goal of this experiment is to assess the prevalence of reducible time series. The purpose is to check to which extent the study of semi-metric quotient spaces is practically justified. For this, we computed the condensed form of every time series by collapsing consecutive replicates to singletons. A time series is reducible if it is longer than its condensed form. For every dataset, we recorded the percentage of reducible time series and the average number of deleted duplicates over the subset of reducible time series. 

\medskip

\begin{figure*}
\begin{minipage}{\textwidth}
\tiny
\centering
\begin{tabular}{l@{\quad}rr@{\quad}rrrcl@{\quad}rr@{\quad}rrr}
\toprule
Data & $\ell$ & $n$ & $p_{\text{red}}$ & $\mu_{\text{del}}$ & $\sigma_{\text{del}}$ 
& & Data & $\ell$ & $n$ & $p_{\text{red}}$ & $\mu_{\text{del}}$ & $\sigma_{\text{del}}$\\
\midrule
50Words & 270 & 905 & 8.1 & 1.1 & 0.4 & & MedicalImages & 99 & 1141 & 2.0 & 1.0 & 0.0\\
Adiac & 176 & 781 & 5.1 & 1.1 & 0.3 & & MiddlePhalOutAgeGroup & 80 & 554 & 40.3 & 28.4 & 4.1\\
ArrowHead & 251 & 211 & 76.8 & 6.2 & 6.3 & & MiddlePhalOutCorrect & 80 & 891 & 36.7 & 28.2 & 4.8\\
Beef & 470 & 60 & 31.7 & 1.2 & 0.4 & & MiddlePhalanxTW & 80 & 553 & 40.3 & 28.4 & 4.1\\
BeetleFly & 512 & 40 & 27.5 & 1.3 & 0.6 & & MoteStrain & 84 & 1272 & 23.7 & 6.3 & 14.8\\
BirdChicken & 512 & 40 & 37.5 & 16.7 & 26.0 & & NonInvasiveFatalECGThorax1 & 750 & 3765 & 97.1 & 11.4 & 12.2\\
Car & 577 & 120 & 86.7 & 1.6 & 0.8 & & NonInvasiveFatalECGThorax2 & 750 & 3765 & 98.4 & 11.1 & 9.2\\
CBF & 128 & 930 & 0.43 & 1.0 & 0.0 & & OliveOil & 570 & 60 & 15 & 1.1 & 0.3\\
ChlorineConcentration & 166 & 4307 & 5.2 & 20.5 & 21.2 & & OSULeaf & 427 & 442 & 5.2 & 1.3 & 0.6\\
CinCECGtorso & 1639 & 1420 & 72.2 & 114.0 & 155.0 & & PhalOutCorrect & 80 & 2658 & 44.9 & 28.3 & 4.3\\
Coffee & 286 & 56 & 7.1 & 1.0 & 0.0 & & Phoneme & 1024 & 2110 & 7.9 & 16.8 & 67.6\\
Computers & 720 & 500 & 99.6 & 513.9 & 168.7 & & Plane & 144 & 210 & 0 & 0.0 & 0.0\\
CricketX & 300 & 780 & 6.2 & 1.0 & 0.2 & & ProximalPhalOutAgeGroup & 80 & 605 & 46.1 & 28.0 & 5.2\\
CricketY & 300 & 780 & 1.8 & 1.1 & 0.3 & & ProximalPhalOutCorrect & 80 & 891 & 54.0 & 28.4 & 4.2\\
CricketZ & 300 & 780 & 5.9 & 1.1 & 0.2 & & ProximalPhalanxTW & 80 & 605 & 46.1 & 28.0 & 5.2\\
DiatomSizeReduction & 345 & 322 & 88.5 & 1.4 & 0.7 & & RefrigerationDevices & 720 & 750 & 99.5 & 463.5 & 127.3\\
DistalPhalOutAgeGroup & 80 & 539 & 49.5 & 28.5 & 3.8 & & ScreenType & 720 & 750 & 99.7 & 560.9 & 139.5\\
DistalPhalOutCorrect & 80 & 876 & 44.1 & 28.4 & 4.0 & & ShapeletSim & 500 & 200 & 1 & 1.0 & 0.0\\
DistalPhalanxTW & 80 & 539 & 49.6 & 28.5 & 3.8 & & ShapesAll & 512 & 1200 & 63.6 & 15.2 & 43.5\\
Earthquakes & 512 & 461 & 99.6 & 351.0 & 60.2 & & SmallKitchenAppliances & 720 & 750 & 99.7 & 625.4 & 154.9\\
ECG200 & 96 & 200 & 1 & 2.0 & 1.0 & & SonyAIBORobotSurface1 & 70 & 621 & 99.7 & 17.1 & 3.6\\
ECG5000 & 140 & 5000 & 1.3 & 4.1 & 5.7 & & SonyAIBORobotSurface2 & 65 & 980 & 99.7 & 14.5 & 3.9\\
ECGFiveDays & 136 & 884 & 99.6 & 6.7 & 2.7 & & StarLightCurves & 1024 & 9236 & 96.5 & 4.9 & 4.0\\
ElectricDevices & 96 & 16637 & 98.3 & 57.5 & 30.0 & & Strawberry & 235 & 983 & 95.3 & 1.7 & 0.8\\
FaceAll & 131 & 2250 & 54.4 & 2.4 & 2.1 & & SwedishLeaf & 128 & 1125 & 10.9 & 1.1 & 0.3\\
FaceFour & 350 & 112 & 98.2 & 166.3 & 19.9 & & Symbols & 398 & 1020 & 99.1 & 45.9 & 26.2\\
FacesUCR & 131 & 2250 & 54.0 & 2.4 & 2.1 & & Synthetic\_Control & 60 & 600 & 0 & 0.0 & 0.0\\
Fish & 463 & 350 & 17.7 & 1.1 & 0.3 & & ToeSegmentation1 & 277 & 268 & 4.1 & 1.0 & 0.0\\
FordA & 500 & 4921 & 45.9 & 2.4 & 1.4 & & ToeSegmentation2 & 343 & 166 & 26.5 & 45.6 & 31.4\\
FordB & 500 & 4446 & 2.4 & 1.0 & 0.1 & & Trace & 275 & 200 & 7.5 & 1.1 & 0.2\\
GunPoint & 150 & 200 & 33 & 2.3 & 0.5 & & TwoLeadECG & 82 & 1162 & 99.2 & 6.3 & 3.1\\
Ham & 431 & 214 & 98.1 & 12.2 & 6.4 & & TwoPatterns & 128 & 5000 & 99.7 & 44.1 & 6.8\\
HandOutlines & 2709 & 1370 & 99.9 & 34.2 & 31.5 & & UWaveGestureLibraryAll & 945 & 4478 & 99.9 & 267.6 & 108.4\\
Haptics & 1092 & 463 & 62.2 & 1.6 & 0.9 & & UWaveGestureLibraryX & 315 & 4478 & 99.9 & 84.0 & 40.4\\
Herring & 512 & 128 & 20.3 & 1.2 & 0.4 & & UWaveGestureLibraryY & 315 & 4478 & 99.9 & 101.6 & 38.2\\
InlineSkate & 1882 & 650 & 87.7 & 10.9 & 66.0 & & UWaveGestureLibraryZ & 315 & 4478 & 99.9 & 82.0 & 41.0\\
InsectWingbeatSound & 256 & 2200 & 27.1 & 1.8 & 1.3 & & Wafer & 152 & 7164 & 93.3 & 80.9 & 25.8\\
ItalyPowerDemand & 24 & 1096 & 60.3 & 1.6 & 0.8 & & Wine & 234 & 111 & 79.3 & 1.7 & 0.8\\
LargeKitchenAppliances & 720 & 750 & 99.7 & 591.0 & 153.4 & & WordSynonyms & 270 & 905 & 8.1 & 1.1 & 0.4\\
Lighting2 & 637 & 121 & 98.3 & 108.3 & 11.0 & & Worms & 900 & 258 & 98.8 & 47.2 & 49.2\\
Lighting7 & 319 & 143 & 98.6 & 54.0 & 5.3 & & WormsTwoClass & 900 & 258 & 98.8 & 47.2 & 49.2\\
Mallat & 1024 & 2400 & 64.5 & 1.7 & 0.9 & & Yoga & 426 & 3300 & 46.2 & 1.3 & 0.5\\
Meat & 448 & 120 & 22.5 & 1.0 & 0.2 & & & & & & & \\
\midrule
\textbf Average & & & & & & & & \textbf 418.1 & \textbf 1597.6 & \textbf 67.8 & \textbf 73.8 & \textbf 30.1\\
\bottomrule
\end{tabular}
\captionof{table}{Results of condensation. Legend: $\ell = $ length of time series $\bullet$ $n = $ number of time series $\bullet$ $p_{\text{red}} = $ percentage of reducible time series $\bullet$ $\mu_{\text{del}} = $ average number of deleted elements $\bullet$ $\sigma_{\text{del}} = $ standard deviation of deleted elements.}
\label{tab:condensation}
\end{minipage}

\vspace{2cm}

\begin{minipage}{\textwidth}
\includegraphics[width=0.4\textwidth]{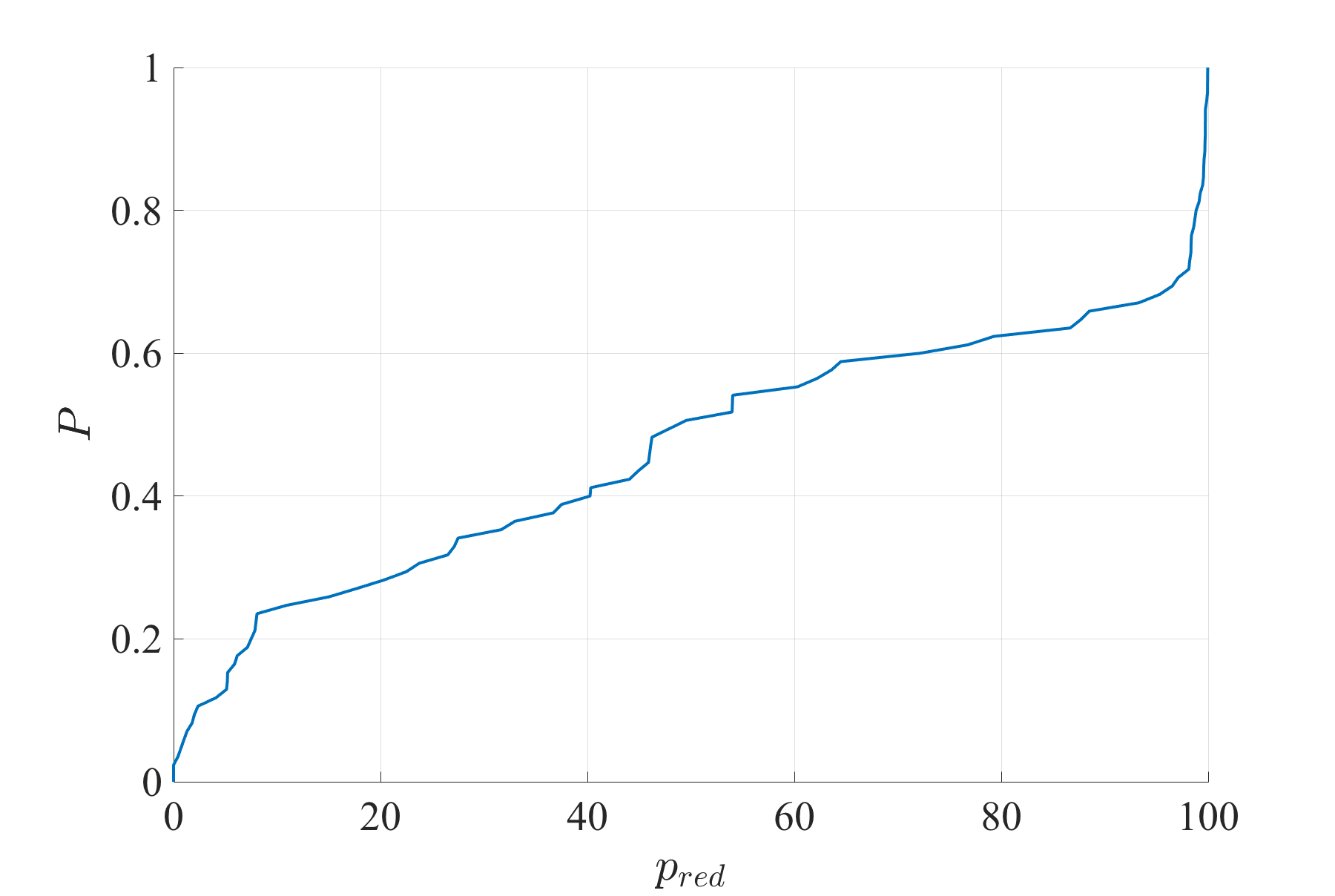}
\hfill
\includegraphics[width=0.4\textwidth]{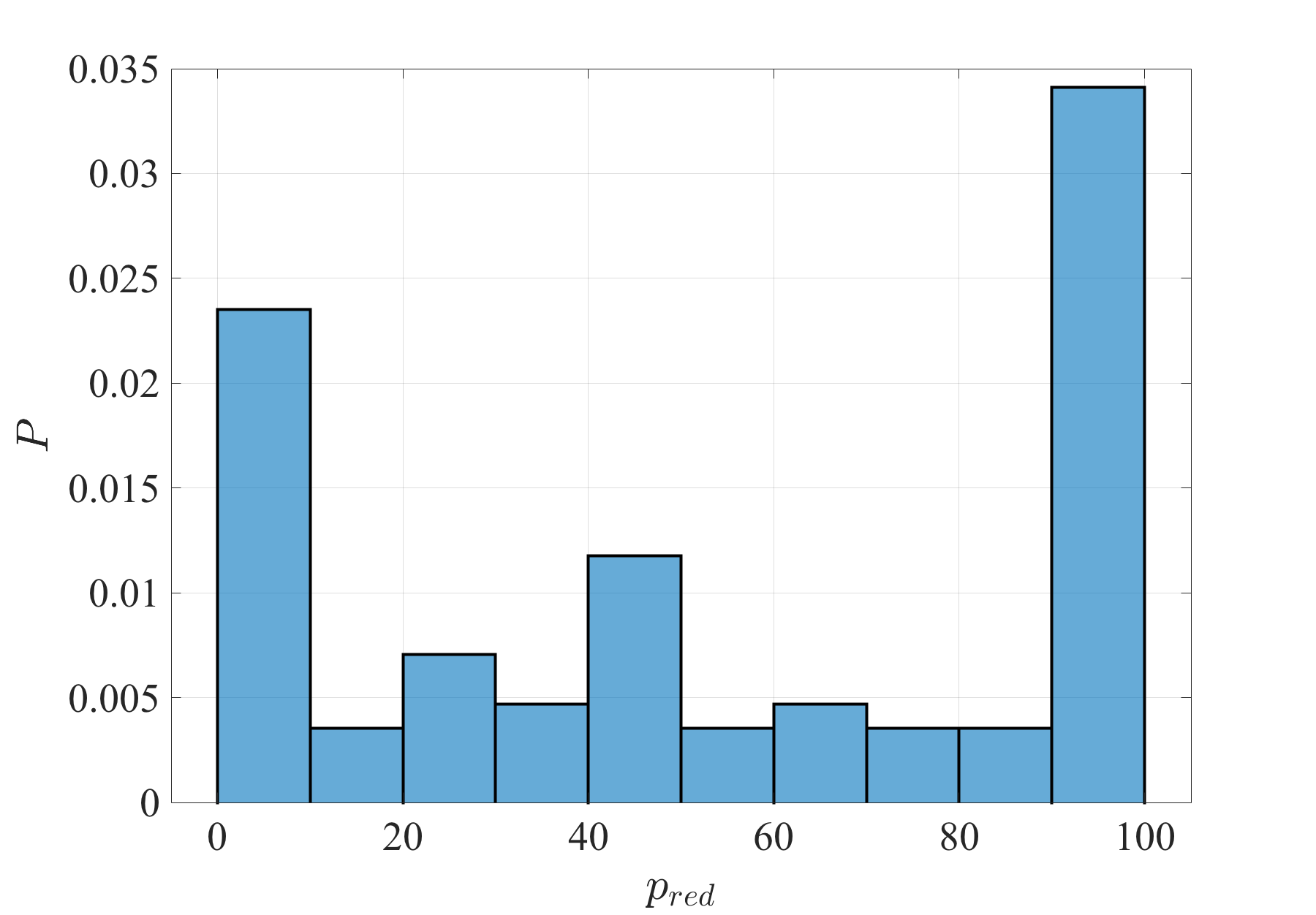}
\captionof{figure}{Estimated cumulative distribution function of the percentage $p_{\text{red}}$ of reducible time series (left) and a histogram as an estimate of the probability density of $p_{\text{red}}$ (right).}
\label{fig:pdf}
\end{minipage}
\end{figure*}

Table \ref{tab:condensation} summarizes the results. The last line shows the weighted averages of every numerical column over the $85$ datasets. The averages are weighted by the sizes $n$ of the datasets. 

The results show that on average $67.8 \%$ of all time series are reducible. There are only two datasets that contain no reducible time series, five datasets with at most $1 \%$ reducible time series, and $17$ datasets with at least $99 \%$ reducible time series. Figure \ref{fig:pdf} shows the estimated cumulative distribution function and a histogram of the percentage $p_{\text{red}}$ of reducible time series. These results indicate that reducible time series occur frequently and therefore justify to study data mining methods on the proposed quotient space $\args{\S{T}^*, \delta^*}$. 

The Pearson correlation coefficient $\rho$ and Spearman rank correlation coefficient $r$ between the length $\ell$ of time series and the percentage $p_{\text{del}}$ of reducible time series is $\rho = 0.32$ and $r = 0.28$, respectively. These results indicate a weak positive correlation between the length of time series and the percentage of reducible time series. The respective p-values $p_\rho = 0.0032$ and $p_r = 0.0086$ suggest that the the correlations are significantly different from zero. As expected, this finding suggests that it is more likely to encounter consecutive duplicates in longer time series.

Condensing reducible time series shortens their lengths by approximately $28.0 \%$ on average.\footnote{For every dataset $d$ compute the number $n_{\text{red}}^d = n \cdot p_{\text{red}}$ of reducible time series and the average percentage $p_{\text{del}}^d = 100 \cdot \mu_{\text{del}}/\ell$ of length difference between the original reducible time series and their condensed forms. Then compute $\sum_d p_{\text{del}}^d \,/\, \sum_d n_{\text{red}}^d$ to obtain the average percentage of length difference over all datasets.} This implies that dtw-comparisons of irreducible time series with a condensed form $x^*$ are on average $1.4$-times faster than with the corresponding reducible time series $x$. This result indicates that computing the quotient distance $\delta^*$ 
gives a slight speed advantage over computing the dtw-distance $\delta$.

\subsection{Nearest-Neighbor Classification}
In this experiment, we compare the classification accuracies of the nearest-neighbor (nn) classifiers using the dtw-distance $\delta$ and the proposed semi-metric $\delta^*$. 

\medskip

The nn-classifiers used the training examples as prototypes and the test examples for estimating the classification accuracy. To apply the $\dtw^*$-nn classifier, we transformed all reducible time series to their condensed forms. For every dataset, we recorded the classification accuracy \emph{acc} of the $\dtw$-nn classifier, the classification accuracy \emph{acc}$^*$ of the $\dtw^*$-nn classifier, and the error percentage $\text{err} = 100 \cdot (\text{acc} - \text{acc}^*)/\text{acc}$. Positive (negative) error percentages mean that the accuracy of the $\dtw$-nn classifier is higher (lower) than the accuracy of the $\dtw^*$-nn classifier.

Table \ref{tab:nn-classifier} presents the classification accuracies of both classifiers. The results show an accuracy-record of $30$ wins ($35.3\%$), $29$ ties ($34.1 \%$), and $26$ losses ($30.6 \%$) of the $\dtw^*$-nn classifier giving a winning percentage of $w^* = (30+0.5\cdot 29)/85 = 52.4$. This finding indicates that both classifiers are comparable with slight advantages for the 
$\dtw^*$-nn classifier. On average, the error percentage is $0.14 \%$. This result also suggests that both classifiers are comparable but with slight advantages for the $\dtw$-nn classifier. The Wilcoxon signed rank test fails to reject the null hypothesis that the differences $\text{acc} - \text{acc}^*$ come from a distribution whose median is zero at significance level $\alpha = 0.05$. The corresponding p-value of the test is $0.87$. This result indicates that the differences in accuracy of both classifiers are not statistically significant. Figure \ref{fig:scatter} visually confirms that the classification accuracies of both classifiers are comparable. 

To conclude, the results suggest to study data mining methods such as k-means, learning vector quantization, and self-organizing maps on condensed time series to overcome the peculiarities caused by the dtw-distance as discussed in Section \ref{sec:peculiarities}. 

\begin{figure*}
\begin{minipage}{\textwidth}
\tiny
\centering
\begin{tabular}{l@{\qquad}rrrcl@{\qquad}rrr}
\toprule
Data & acc & acc$^*$ & err && Data & acc & acc$^*$ & err \\
\midrule
50Words & 69.0 & 69.0 & 0.00 & & MedicalImages & 73.7 & 73.7 & 0.00 \\
Adiac & 60.4 & 59.8 & 0.85 & & MiddlePhalanxOutlineAgeGroup & 50.0 & 48.7 & 2.60 \\
ArrowHead & 70.3 & 72.0 & -2.44 & & MiddlePhalanxOutlineCorrect & 69.8 & 70.1 & -0.49 \\
Beef & 63.3 & 63.3 & 0.00 & & MiddlePhalanxTW & 50.6 & 50.6 & 0.00 \\
BeetleFly & 70.0 & 70.0 & 0.00 & & MoteStrain & 83.5 & 83.8 & -0.38 \\
BirdChicken & 75.0 & 75.0 & 0.00 & & NonInvasiveFatalECGThorax1 & 79.0 & 79.3 & -0.39 \\
Car & 73.3 & 73.3 & 0.00 & & NonInvasiveFatalECGThorax2 & 86.5 & 86.4 & 0.12 \\
CBF & 99.7 & 99.7 & 0.00 & & OliveOil & 83.3 & 86.7 & -4.00 \\
ChlorineConcentration & 64.8 & 64.8 & 0.00 & & OSULeaf & 59.1 & 59.5 & -0.70 \\
CinCECGtorso & 65.1 & 63.5 & 2.45 & & PhalangesOutlinesCorrect & 72.8 & 73.4 & -0.80 \\
Coffee & 100.0 & 100.0 & 0.00 & & Phoneme & 22.8 & 22.8 & 0.00 \\
Computers & 70.0 & 63.2 & 9.71 & & Plane & 100.0 & 100.0 & 0.00 \\
CricketX & 75.4 & 75.1 & 0.34 & & ProximalPhalanxOutlineAgeGroup & 80.5 & 80.5 & 0.00 \\
CricketY & 74.4 & 74.4 & 0.00 & & ProximalPhalanxOutlineCorrect & 78.4 & 80.1 & -2.19 \\
CricketZ & 75.4 & 75.4 & 0.00 & & ProximalPhalanxTW & 76.1 & 76.1 & 0.00 \\
DiatomSizeReduction & 96.7 & 96.7 & 0.00 & & RefrigerationDevices & 46.4 & 47.7 & -2.87 \\
DistalPhalanxOutlineAgeGroup & 77.0 & 78.4 & -1.87 & & ScreenType & 39.7 & 40.5 & -2.01 \\
DistalPhalanxOutlineCorrect & 71.7 & 70.7 & 1.52 & & ShapeletSim & 65.0 & 64.4 & 0.86 \\
DistalPhalanxTW & 59.0 & 59.7 & -1.22 & & ShapesAll & 76.8 & 76.7 & 0.22 \\
Earthquakes & 71.9 & 69.8 & 3.00 & & SmallKitchenAppliances & 64.3 & 72.8 & -13.28 \\
ECG200 & 77.0 & 77.0 & 0.00 & & SonyAIBORobotSurface1 & 72.5 & 71.5 & 1.38 \\
ECG5000 & 92.4 & 92.5 & -0.02 & & SonyAIBORobotSurface2 & 83.1 & 82.8 & 0.38 \\
ECGFiveDays & 76.8 & 76.7 & 0.15 & & StarLightCurves & 90.7 & 90.7 & -0.09 \\
ElectricDevices & 60.1 & 55.9 & 7.00 & & Strawberry & 94.0 & 94.0 & 0.00 \\
FaceAll & 80.8 & 80.4 & 0.51 & & SwedishLeaf & 79.2 & 79.2 & 0.00 \\
FaceFour & 83.0 & 83.0 & 0.00 & & Symbols & 95.0 & 93.9 & 1.16 \\
FacesUCR & 90.5 & 90.4 & 0.05 & & Synthetic\_Control & 99.3 & 99.3 & 0.00 \\
Fish & 82.3 & 82.3 & 0.00 & & ToeSegmentation1 & 77.2 & 77.2 & 0.00 \\
FordA & 56.2 & 56.2 & -0.05 & & ToeSegmentation2 & 83.8 & 85.4 & -1.84 \\
FordB & 59.4 & 59.5 & -0.09 & & Trace & 100.0 & 100.0 & 0.00 \\
GunPoint & 90.7 & 92.0 & -1.47 & & TwoLeadECG & 90.5 & 87.5 & 3.30 \\
Ham & 46.7 & 47.6 & -2.04 & & TwoPatterns & 100.0 & 99.7 & 0.35 \\
HandOutlines & 79.8 & 78.7 & 1.38 & & UWaveGestureLibraryAll & 89.2 & 91.6 & -2.72 \\
Haptics & 37.7 & 38.6 & -2.59 & & UWaveGestureLibraryX & 72.8 & 71.1 & 2.26 \\
Herring & 53.1 & 54.7 & -2.94 & & UWaveGestureLibraryY & 63.4 & 63.6 & -0.31 \\
InlineSkate & 38.4 & 38.2 & 0.47 & & UWaveGestureLibraryZ & 65.8 & 66.5 & -1.06 \\
InsectWingbeatSound & 35.5 & 35.6 & -0.14 & & Wafer & 98.0 & 98.4 & -0.38 \\
ItalyPowerDemand & 95.0 & 92.3 & 2.86 & & Wine & 57.4 & 57.4 & 0.00 \\
LargeKitchenAppliances & 79.5 & 71.5 & 10.07 & & WordSynonyms & 64.9 & 64.9 & 0.00 \\
Lighting2 & 86.9 & 88.5 & -1.89 & & Worms & 46.4 & 46.4 & 0.00 \\
Lighting7 & 72.6 & 68.5 & 5.66 & & WormsTwoClass & 66.3 & 63.5 & 4.17 \\
Mallat & 93.4 & 93.4 & -0.05 & & Yoga & 83.7 & 83.8 & -0.16 \\
Meat & 93.3 & 93.3 & 0.00 & & & & & \\
\midrule
\textbf Average & & & & & & & & \textbf 0.14 \\
\bottomrule
\end{tabular}
\captionof{table}{Result of nearest-neighbor classification. Legend: acc $=$ accuracy of the $\dtw$-nn classifier $\bullet$ acc$^* =$ accuracy of the $\dtw^*$-nn classifier $\bullet$ err $=$ error percentage $100 \cdot (\text{acc}-\text{acc}^*) / \text{acc}$.}
\label{tab:nn-classifier}
\end{minipage}

\vspace{2cm}

\begin{minipage}{\textwidth}
\includegraphics[width=0.45\textwidth]{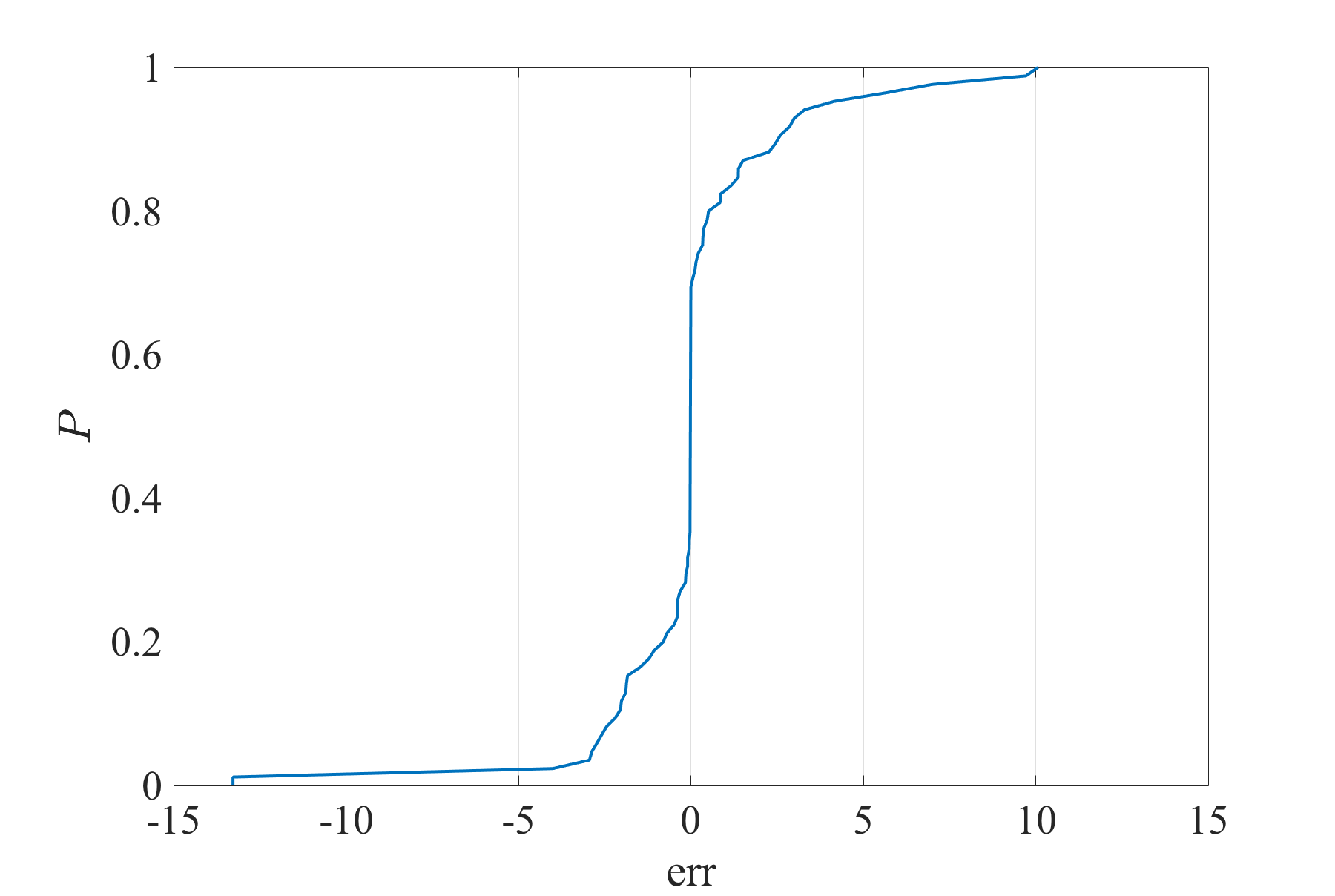}
\hfill
\includegraphics[width=0.45\textwidth]{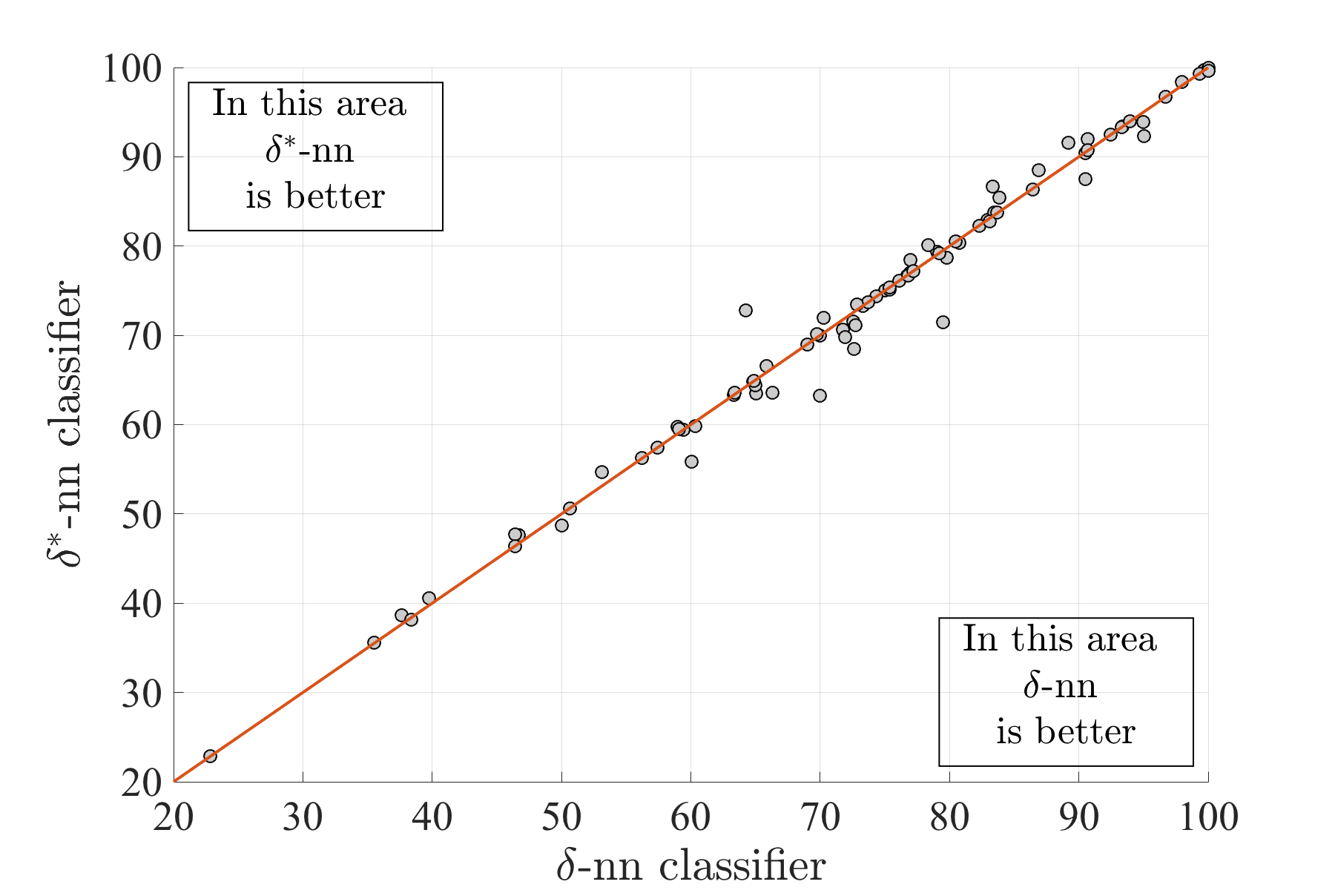}
 \captionof{figure}{Estimated cumulative probability distribution function of the error percentage \emph{err} (left) and scatterplot of classification accuracies of the $\delta$-nn and $\delta^*$-nn classifier (right).}
\label{fig:scatter}
\end{minipage}
\end{figure*}

\section{Conclusion}\label{sec:conclusion}

The dtw-distance fails to satisfy the triangle inequality and the identity of indiscernibles. As a consequence, the dtw-distance is not warping-invariant, although it has been designed to eliminate temporal variations. Lack of warping-invariance of the dtw-distance results in peculiarities of data mining tasks in dtw-spaces. To overcome these peculiarities, this article converts the dtw-distance into a semi-metric whose canonical extension is warping-invariant. Empirical results on nearest-neighbor classification in the proposed semi-metric space and the original dtw-space show that the respective classification accuracies are comparable. This finding suggests to further explore data mining applications in the semi-metric spaces induced by the dtw-distance.



\clearpage
\begin{thebibliography}{00}
\setlength{\parskip}{0pt}
\setlength{\itemsep}{0pt plus 0.3ex}
\small

\bibitem{Abanda2018}
A.~Abanda, U.~Mori, and J.A.~Lozano. 
\newblock A review on distance based time series classification. 
\newblock arXiv:1806.04509, 2018.

\bibitem{Abdulla2003}
W.H.~Abdulla, D.~Chow, and G.~Sin. 
\newblock Cross-words reference template for DTW-based speech recognition systems. 
\newblock \emph{Conference on Convergent Technologies for Asia-Pacific Region}, 2003.

\bibitem{Aghabozorgi2015}
S.~Aghabozorgi, A.S.~Shirkhorshidi, and T.-Y.~Wah.
\newblock Time-series clustering -- A decade review.
\newblock \emph{Information Systems}, 53:16--38, 2015.

\bibitem{Bagnall2017}
A.~Bagnall, J.~Lines, A.~Bostrom, J.~Large, and E.~Keogh. 
\newblock The great time series classification bake off: a review and experimental evaluation of recent algorithmic advances. 
\newblock \emph{Data Mining and Knowledge Discovery}, 31(3):606--660, 2017.

\bibitem{Bagnall2018}
A.~Bagnall, J.~Lines, W.~Vickers, and E.~Keogh.
\newblock \emph{The UEA \& UCR Time Series Classification Repository}, 
\newblock \url{www.timeseriesclassification.com}. Accessed: 08/2018.

\bibitem{Bulteau2018}
L.~Bulteau, V.~Froese, and R.Ñiedermeier.
\newblock Hardness of Consensus Problems for Circular Strings and Time Series Averaging.
\newblock \emph{CoRR}, abs/1804.02854, 2018.

\bibitem{Casacuberta1987}
F.~Casacuberta, E.~Vidal, and H.~Rulot.
\newblock On the metric properties of dynamic time warping. 
\newblock \emph{IEEE Transactions on Acoustics, Speech, and Signal Processing}, 35(11):1631--1633, 1987.

\bibitem{Cuturi2011}
M.~Cuturi. 
\newblock Fast global alignment kernels. 
\newblock \emph{International Conference on Machine Learning} (ICML~'11), 2011.

\bibitem{Cuturi2017}
M.~Cuturi and M.~Blondel.
\newblock Soft-DTW: A Differentiable Loss Function for Time-Series.
\newblock \emph{International Conference on Machine Learning} (ICML~'17), 2017.

\bibitem{Esling2012}
P.~Esling and C.~Agon. 
\newblock Time-series data mining. 
\newblock \emph{ACM Computing Surveys}, 45(1), 2012.

\bibitem{Fu2011}
T.-C.~Fu. 
\newblock A review on time series data mining.
\newblock \emph{Engineering Applications of Artificial Intelligence}, 24(1):164--181, 2011.

\bibitem{Frechet1948}
M.~Fr\'{e}chet.
\newblock Les \'el\'ements al\'eatoires de nature quelconque dans un espace distanci\'e.
\newblock \emph{Annales de l'institut Henri Poincar\'e}, 215--310, 1948.

\bibitem{Hautamaki2008}
V.~Hautamaki, P.~Nykanen, P.~Franti.
\newblock Time-series clustering by approximate prototypes.
\newblock \emph{International Conference on Pattern Recognition}, 2008.

\bibitem{Jain2016b}
B.J.~Jain and D.~Schultz.
\newblock On the Existence of a Sample Mean in Dynamic Time Warping Spaces.
\newblock arXiv:1610.04460, 2016.

\bibitem{Jain2018}
B.J.~Jain and D.~Schultz.
\newblock Asymmetric learning vector quantization for efficient nearest neighbor classification in dynamic time warping spaces
\newblock \emph{Pattern Recognition} 76, 349--366, 2018.

\bibitem{Kohonen1998} 
T.~Kohonen and P.~Somervuo.
\newblock Self-organizing maps of symbol strings. 
\newblock \emph{Neurocomputing}, 21(1-3):19--30, 1998.

\bibitem{Kruskal1983}
J.B.~Kruskal and M.~Liberman. 
\newblock The symmetric time-warping problem: From continuous to discrete.
\newblock \emph{Time warps, string edits and macromolecules: The theory and practice of sequence comparison}, 1983.

\bibitem{Lemire2009}
D.~Lemire. 
\newblock Faster retrieval with a two-pass dynamic-time-warping lower bound. 
\newblock \emph{Pattern Recognition}, 42(9):2169--2180, 2009.

\bibitem{Marteau2009}
P.F.~Marteau. 
\newblock Time warp edit distance with stiffness adjustment for time series matching. 
\newblock \emph{IEEE Transactions on Pattern Analysis and Machine Intelligence}, 31(2):306--318, 2009.

\bibitem{Morel2018}
M.~Morel, C.~Achard, R.~Kulpa, and S.~Dubuisson.
\newblock Time-series Averaging Using Constrained Dynamic Time Warping with Tolerance. 
\newblock \emph{Pattern Recognition}, 74, 2018.

\bibitem{Mueller2007}
M.~M\"uller.
Dynamic time warping
\emph{Information retrieval for music and motion}, 69--84, 2007.

\bibitem{Petitjean2011}
F.~Petitjean, A.~Ketterlin, and P.~Gancarski. 
\newblock A global averaging method for dynamic time warping, with applications to clustering.
\newblock \emph{Pattern Recognition} 44(3):678--693, 2011.

\bibitem{Petitjean2016}
F.~Petitjean, G.~Forestier, G.I.~Webb, A.E.~Nicholson, Y.~Chen, and E.~Keogh.
\newblock Faster and more accurate classification of time series by exploiting a novel dynamic time warping averaging algorithm. 
\newblock \emph{Knowledge and Information Systems}, 47(1):1--26, 2016.

\bibitem{Rabiner1979}
L.R.~Rabiner and J.G. Wilpon.
\newblock Considerations in applying clustering techniques to speaker-independent word recognition. 
\newblock \emph{The Journal of the Acoustical Society of America}, 66(3): 663--673, 1979.

\bibitem{Sakoe1978}
H.~Sakoe and S.~Chiba. 
\newblock Dynamic programming algorithm optimization for spoken word recognition. 
\newblock \emph{IEEE Transactions on Acoustics, Speech, and Signal Processing}, 26(1):43--49, 1978.

\bibitem{Schultz2018}
D.~Schultz and B.~Jain.
\newblock Nonsmooth analysis and subgradient methods for averaging in dynamic time warping spaces.
\newblock \emph{Pattern Recognition}, 74, 2018.

\bibitem{Soheily-Khah2016}
S.~Soheily-Khah, A.~Douzal-Chouakria, and E.~Gaussier.
\newblock Generalized k-means-based clustering for temporal data under weighted and kernel time warp.
\newblock \emph{Pattern Recognition Letters}, 75:63--69, 2016.

\bibitem{Somervuo1999} 
P. Somervuo and T. Kohonen,
\newblock Self-organizing maps and learning vector quantization for feature sequences.
\newblock \emph{Neural Processing Letters}, 10(2):151--159, 1999.
\end{thebibliography}
\end{document}